\documentclass{article}

\usepackage{graphicx} % more modern
\usepackage{subfigure} 

\usepackage{natbib}

\usepackage{algorithm}
\usepackage{algorithmic}

\usepackage{amssymb,amsmath,dsfont,multirow}
\usepackage{amsthm}
\usepackage{thmtools}
\usepackage{tikz}
\usetikzlibrary{calc,intersections,through,positioning,matrix}
\usetikzlibrary{decorations.markings}
\usepackage{hyperref}

\usepackage[accepted]{icml2013}

\icmltitlerunning{Margins, Shrinkage, and Boosting}

\usepackage{cleveref}
\usepackage{mathrsfs}

\def\R{\mathbb R}

\def\bG{\mathbb G}
\def\bL{\mathbb L}

\def\cH{\mathcal H}
\def\cL{\mathcal L}
\def\cM{\mathcal M}
\def\cO{\mathcal O}

\def\cX{\mathcal X}

\def\sfP{\mathsf P}

\def\im{\textup{im}}

\def\bfe{\textbf e}

\newcommand{\1}{\mathds{1}}

\newcommand\alqubT[2]{\alpha_{#1}^{\textup{Q}}(#2)}
\newcommand\aloptT[2]{\alpha_{#1}^{\textup{O}}(#2)}
\newcommand\alwolT[2]{\alpha_{#1}^{\textup{W}}(#2)}
\newcommand\aladaT[2]{\alpha_{#1}^{\textup{A}}(#2)}

\newcommand\alqub[1]{\alqubT{#1}{\nu}}
\newcommand\alopt[1]{\aloptT{#1}{\nu}}
\newcommand\alwol[1]{\alwolT{#1}{\nu}}
\newcommand\alada[1]{\aladaT{#1}{\nu}}

\newcommand{\argmax}{\operatornamewithlimits{arg\,max}}

\numberwithin{equation}{section}
\declaretheorem[numberlike=equation]{theorem}
\declaretheorem[numberlike=theorem]{lemma}
\declaretheorem[numberlike=theorem]{proposition}

\declaretheoremstyle[%
qed={\ensuremath\Diamond}]{remstyle}
\declaretheorem[numberlike=theorem,style=remstyle]{definition}

\begin{document}

\twocolumn[
\icmltitle{Margins, Shrinkage, and Boosting}

% unless you've provided the [accepted] option to the icml2013
\icmlauthor{Matus Telgarsky}{mtelgars@cs.ucsd.edu\hspace{-0.495cm}}
\icmladdress{Department of Computer Science and Engineering, UCSD,
9500 Gilman Drive, La Jolla, CA 92093-0404}

\icmlkeywords{margins, shrinkage, boosting}

\vskip 0.3in
]

\begin{abstract}
    This manuscript shows that AdaBoost and its immediate variants can produce
    approximate maximum margin classifiers simply by scaling step size choices with
    a fixed small constant.  In this way, when the unscaled step size is an
    optimal choice, these results provide guarantees for Friedman's empirically
    successful ``shrinkage'' procedure for gradient boosting
    \citep{friedman_gradient_boosting}.
    Guarantees are
    also provided for a variety of other step sizes,
    affirming the intuition
    that increasingly regularized line searches provide improved margin
    guarantees.  The results hold for the exponential loss and similar losses,
    most notably the logistic loss.
\end{abstract} 

\section{Introduction}

AdaBoost and related boosting algorithms greedily aggregate many simple predictors
into a single accurate predictor~\citep{freund_schapire_adaboost}.  One explanation
for the efficacy of boosting is that it not only seeks aggregates with low
empirical risk,
but moreover that it prefers good margins, which leads to improved
generalization \citep{boosting_margin}.  Since AdaBoost does not attain maximum margins on
general instances, a push was made to develop methods which carry such a
guarantee~\citep{warmuth_maxmargin_boosting,shai_singer_weaklearn_linsep,rudin_smooth_margin}.

This work shows that margin maximization may be achieved by scaling back the step size.  The
intuition for this result is simple (cf. \Cref{fig:MAGIC}):
when (equivalently) considered as steps in a
coordinate descent
procedure, the iterates, depicted as a path, approximate the path of constrained optima
(for all possible choices of constraint).  By scaling back the step size, the optimal
path is more finely approximated.
%This intuition appears in \Cref{fig:MAGIC}.
As there have been many proposed step sizes for these methods, this manuscript will study
four separate choices, deriving improved bounds for the more regularized choices.
While it has been shown before that regularized step sizes have good generalization and
asymptotically good margins~\citep{zhang_yu_boosting},
this manuscript shows that straightforward step choices achieve these margins at rates
matching explicitly margin-maximizing boosting methods.
%   Notably, \citet{zhang_yu_boosting} have already shown that regularized line searches
%   can lead to good margins and consistency; the distinction here is that the step sizes
%   are used in practice, and 
%   The value of less aggressive line search was provided by \citet{zhang_yu_boosting};
%   this manuscript provides guarantees for the simple choice of inserting a single constant.
%   \red{mention that this paper not only gives rates: in the best case, they match those
%   of the max margin boosters}

This practice of scaling back weights was proposed by
\citet[Section 5]{friedman_gradient_boosting}, who referred to it as
a shrinkage scheme \citep{copas_shrinkage}.
This scheme is effective, and adopted in practice
(see for instance \citet[Class \texttt{CvGBTrees}]{opencv_library} and
\citet[Class \texttt{GradientBoostingClassifier}]{scikit-learn});
the purpose of this manuscript is to provide theoretical guarantees.

\begin{figure}[b!]
\begin{center}
%\centerline{\includegraphics[width=\columnwidth]{ls_contours}}
\centerline{\includegraphics[width=\columnwidth]{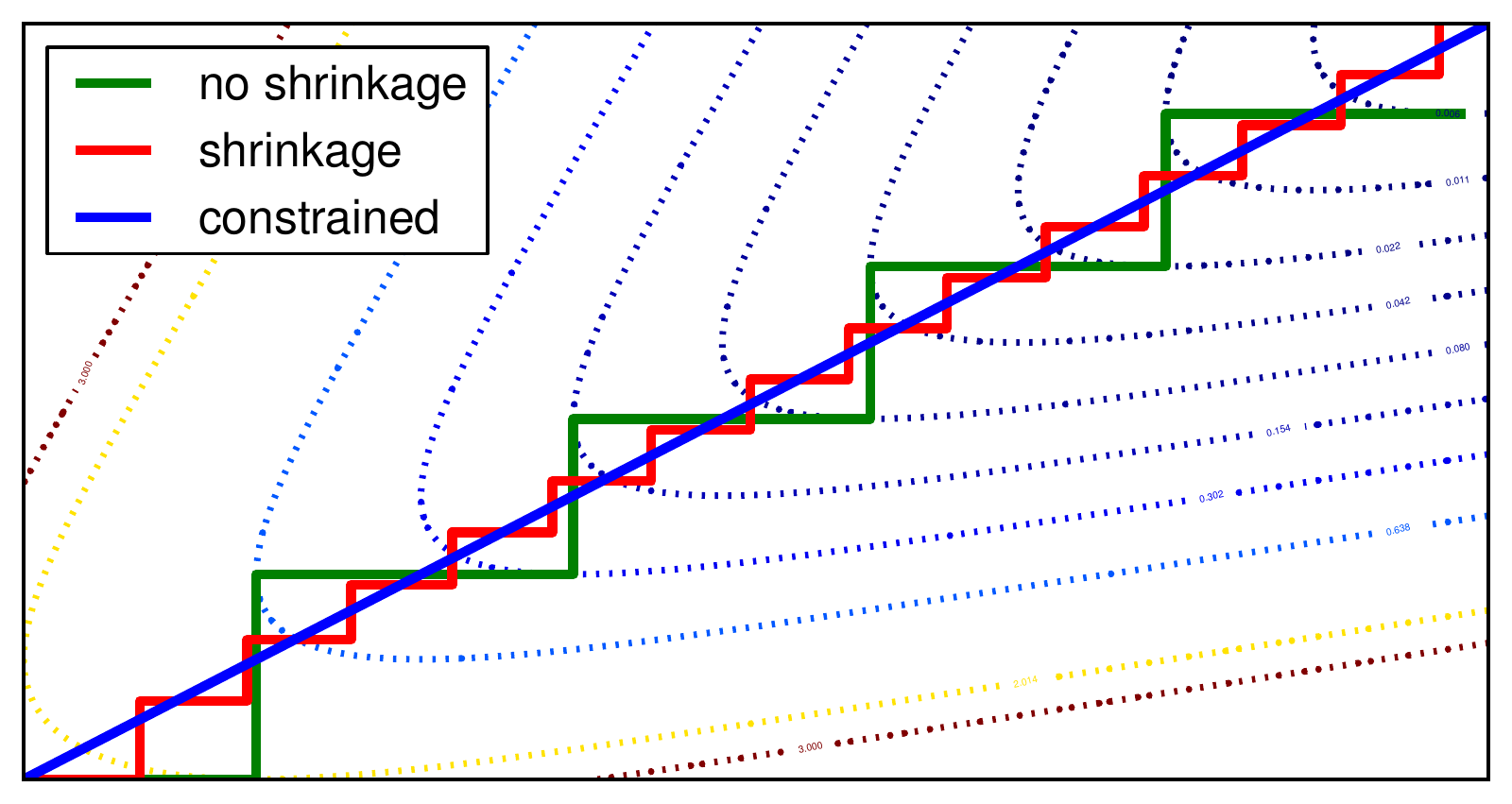}}
\caption{The blue diagonal line is the empirical risk minimizer subject to
    varying $l^1$ constraints, and is also a maximum margin choice.  The green
    line takes optimal steps, and grossly overshoots the optimal path.  By
    applying mild shrinkage, the red line approximates
the maximum margin choice much more finely.}
\label{fig:MAGIC}
\end{center}
\end{figure}

\subsection{Outline}
After summarizing the main content,
this introduction closes with connections to related work;
thereafter,
\Cref{sec:setup} recalls the core algorithm, defines the class of loss functions,
and provides the four step sizes.%four step size choices.

As boosting is generally studied under the weak learning assumption (a separability
condition), the dominant study in this manuscript is also under the condition of separability,
and appears in \Cref{sec:sep}.  The first step is to show that shrinkage does not drastically
change the rate of convergence of the empirical risk under these methods.  The more involved
study is on the topic of margins, and the final subsection compares these bounds to those
of other methods.

General (potentially nonseparable) instances are discussed in \Cref{sec:general}.
Once again, the first step is a convergence rate guarantee, which again matches those
without shrinkage.  This section also demonstrates that, under a certain decomposition
of boosting problems, the algorithm is still achieving margins on a separable
sub-component of the problem.

The manuscript closes with some discussion in \Cref{sec:discussion}.  All proofs
are relegated to appendices (in the supplementary material).

\subsection{Related Work}
Three close works proposed regularized line searches for boosting.
First,
\citet{friedman_gradient_boosting} gave the same scheme as is considered here (albeit with
only the optimal line search); follow-up work has been mainly empirical, and the questions of
convergence rates and margin guarantees do not appear in the literature.
Second, \citet{zhang_yu_boosting} also considered regularized line searches,
but with a goal of proving consistency; margin maximization is proved as a byproduct, and
the analogous results here hold under fewer conditions, and come with rates for the more
stringent step sizes.
A third work, due to \citet{ratsch_soft_margins_adaboost}, also proves margin maximizing
properties of regularized line searches, but again without rates.

As mentioned in the introduction, margin maximization properties of AdaBoost have received
extensive study; an excellent survey of results with pointers to other literature is
provided by \citet[Chapter 5]{schapire_freund_book_final}.
Amongst these, a crucial result, due to \citet{rudin_adaboost_bad_margin}, provides
a concrete input to AdaBoost which yields suboptimal margins (which is used
in \Cref{sec:sep:discussion}); that work also studies the
evolution of these margins as a dynamical system, a topic which will reappear
in \Cref{sec:discussion}.

The primary contribution of this manuscript
is to exhibit margin maximization, thus a natural comparison is to other algorithms
with this same guarantee, for instance the works of
\citet{warmuth_maxmargin_boosting},
\citet{shai_singer_weaklearn_linsep},
and \citet{rudin_smooth_margin}
(or again refer to \citet[Chapter 5, Bibliographic Notes]{schapire_freund_book_final} for a more
extensive summary).
This manuscript will briefly compare with the methods of \citet{shai_singer_weaklearn_linsep},
which subsume some earlier results and match the best guarantees, along with giving
a simple,
general, greedy scheme.
The key distinction between previous work and the present work is firstly that the
algorithmic modifications here are minor (in particular, the form of unregularized empirical
risk minimization is unchanged), and that properties of an existing, widely
used method are discerned (namely, the shrinkage procedure presented
by \citet{friedman_gradient_boosting}).

As is standard in the above works, this manuscript is only concerned with convergence
of empirical quantities.

In order to prove convergence rates, this work relies heavily on techniques due
to \citet{primal_dual_boosting_arxiv}.
In particular, the scheme to prove convergence rates of empirical risk,
detailed properties of splitting out a \emph{hard core} from a boosting instance
(cf. \Cref{sec:general}), and the notion of relative curvature (cf. \Cref{sec:setup:loss})
are all due to \citet{primal_dual_boosting_arxiv}.
The intent of the present manuscript is to establish margin properties, and in this
regard it departs from \citet{primal_dual_boosting_arxiv};
by contrast, the convergence rates of empirical risk presented here are thus trivial,
but included since they did not appear explicitly in the literature.
It is worth mentioning that these methods produce bad constants when applied to the
logistic loss; unfortunately, previous work also suffers in this case
(for instance, the work of
\citet{collins_schapire_singer_adaboost_bregman} provided only convergence of
empirical risk, and not rates).

\section{Algorithms and Notation}
\label{sec:setup}

First some basic notation.  Let $\{(x_i,y_i)\}_{i=1}^m \subseteq \cX\times\{-1,+1\}$ denote
an $m$-point sample.  Take $\cH_0$ to denote the collection of weak learners; it is assumed
that $h\in\cH_0$ satisfies $h(\cX) \subseteq [-1,+1]$, and that $\cH_0$ has some form of
bounded complexity, meaning specifically that the set of vectors
$\{ (h(x_1), \ldots, h(x_m)) : h\in \cH_0\}$ is finite;  this for instance holds if
there is a fixed finite set of outputs from $\cH_0$, e.g., each $h$ is binary.
Consequently, let $\cH = \{h_j\}_{j=1}^n$ denote the effective finite set of hypothesis,
and collect the responses on the sample into a matrix $A\in [-1,+1]^{m\times n}$
with $A_{ij} = -y_i h_j(x_i)$.

Boosting finds a weighting $\lambda \in \R^n$ of $\cH$,
which corresponds to a regressor
$x\mapsto \sum_{j=1}^n \lambda_j h_j(x)$,
and thus a binary classification rule after thresholding.
The corresponding ($l_1$ minimum) margin $\cM(A\lambda)$ over the sample
with respect to $\lambda$ is
\[
    \cM(A\lambda)
    :=
    \min_{i\in [m]} \frac {-\bfe_i^\top A\lambda}{\|\lambda\|_1}
    =
    \min_{i\in [m]} \frac {y_i \sum_{j=1}^n \lambda_j h_j(x_i)}{\|\lambda\|_1}
    .
\]
Let $\gamma$ denote the best (largest) achievable margin;
equivalently \citep{shai_singer_weaklearn_linsep}, $\gamma$ is
the weak learning rate (which justifies the choice of $l_1$ margins):
\begin{align*}
    \gamma
    &:=
    \max_{\substack{\lambda\in \R^n\\ \|\lambda\|_1=1}}
    \cM(A\lambda)
    = \max_{\substack{\lambda\in \R^n\\ \|\lambda\|_1=1}}
    \min_{i\in [m]} -\bfe_i^\top A\lambda
    \\
    &= \min_{w \in \Delta_m} \max_{j\in [n]} \left|\sum_{i=1}^m w_i y_i h_j(x_i)\right|
    = \min_{w \in \Delta_m} \|A^\top w\|_\infty.
\end{align*}
When $\gamma>0$, the instance is considered separable; classically, this
condition is termed the \emph{weak learning assumption}
\citep{kearns_valiant_wl,freund_schapire_adaboost}.

\subsection{The Family of Loss Functions}
\label{sec:setup:loss}
The class $\bL$ will effectively be ``functions similar to the exponential loss''.
Some of this is for analytic convenience, but some of this appears to be essential,
and thus a bit of motivation is appropriate.

Optimization problems typically take advantage of curvature (e.g., strong convexity)
to establish a convergence rate.  The analysis here instead uses a relative form of curvature:
it suffices for, say, the Hessian to not be too small relative to the gap between the current
primal objective value and the primal optimum.  In this sense, the exponential loss is ideal,
as it is a fixed point of the differentiation operator.

\begin{definition}
    Given a loss $\ell : \R\to\R_{++}$
    (where $\R_{++}$ denotes positive reals),
    let $C_\ell(z)\geq 1$ (with potentially $C_\ell(z) = \infty$)
    be the tightest positive constant so that, for every $x\leq z$:
    $C_\ell(z)^{-1} \leq \exp(x) / \ell^{(i)}(x) \leq C_\ell(z)$ for $i\in \{0,1,2\}$
    (the zeroth, first, and second derivatives).
\end{definition}

Since $C_\ell(z)$ is defined to be the tightest constant, it follows that $y\leq z$
implies $C_\ell(y) \leq C_\ell(z)$.

From here, the class of loss functions may be defined.

\begin{definition}
    Let $\bL$ contain all functions $\ell:\R\to \R_+$ which are twice continuously
    differentiable, strictly convex, and have $C_\ell(z) < \infty$ for all $z\in\R$.
    Additionally, if $\lim_{z\to-\infty} C_\ell(z) = 1$, then $\ell \in \bL_\infty$.
\end{definition}

Crucially, the two classes $\bL$ and $\bL_\infty$ both contain the exponential and logistic
losses.
\begin{proposition}
    \label{fact:exp_log_bL}
    $\{x\mapsto \exp(x), x\mapsto \ln(1+\exp(x))\} \subseteq \bL_\infty$.
\end{proposition}

One way to interpret this is to say ``in the limit,
logistic loss is the same as exponential loss''.  Unfortunately, this treatment
of the logistic loss ends up being quite unfair, in the sense that the bounds are
not accurately representative of the behavior of the algorithm (see \Cref{sec:sep:discussion}).
It is, however, unclear how to better deal with the logistic loss.

Lastly, the relevant primal objective function may be defined.
\begin{definition}
    Given $\ell\in\bL$ and vector $z\in\R^m$,
    define $\cL(z) := m^{-1} \sum_{i=1}^m \ell(z_i)$,
    whereby the primal optimization problem for boosting is
    to minimize $\cL(A\lambda)$ over the domain $\R^n$.
    For convenience, define $\bar \cL_A := \inf_{\lambda\in\R^n} \cL(A\lambda)$.
\end{definition}

\subsection{Algorithm}
The algorithm appears in \Cref{alg:alg:alg}.  Before defining the various step sizes, two
more definitions are in order.

\begin{algorithm}[t]
    \caption{$\textsc{boost}.$\\
        \textbf{Input:} loss $\ell$, matrix $A\in[-1,+1]^{m\times n}$.\\
    \textbf{Output:} Weighting sequence $\{\lambda_t\}_{t=0}^\infty$.}
    \label{alg:alg:alg}
    \begin{algorithmic}
        \STATE Initialize $\lambda_0 := 0$.
        \FOR{$t = 1,2,\ldots:$}
\STATE Choose column (weak learner)
\[
j_t
%:= \argmax_{j} | \nabla(f\circ A)(\lambda_{t-1})^\top \bfe_j|.
:= \argmax_{j} | \nabla\cL(A\lambda_{t-1})^\top A \bfe_j|.
\]
%   Note
%   \[
%   j_t= \argmax_j |\nabla f(A\lambda_{t-1})^\top A\bfe_j)|.
%   \]
\STATE Set descent direction $v_t \in \{\pm \bfe_{j_t}\}$, whereby
\[
\nabla\cL(A\lambda_{t-1})^\top A v_t
= -\|\nabla\cL(A\lambda_{t-1})^\top A\|_\infty.
\]
\STATE Find $\alpha_t$ via line search.%; i.e., approximately
%minimize $\alpha \mapsto \cL(A(\lambda_{t-1} + \alpha v_t))$.
%   \[
%   \inf_{\alpha>0} (f\circ A)(\lambda_{t-1} + \alpha v_t).
%   \]
\STATE Update $\lambda_t := \lambda_{t-1} + \alpha_t v_t$.
\ENDFOR
    \end{algorithmic}
\end{algorithm}

\begin{definition}
    For every $t$, define
    $\gamma_t :=
    \|A^\top \nabla \cL(A\lambda_{t-1})\|_\infty / \|\nabla\cL(A\lambda_{t-1})\|_1$.
    (Note that $1 \geq \gamma_t \geq \gamma$.)
\end{definition}

Additionally, rather than depending on parameter $C_\ell(z)$ for a carefully chosen $z$, the following definition suffices.
\begin{definition}
    For $t\geq 1$,
    define $C_t := C_\ell(\ell^{-1}(m\cL(A\lambda_{t-1})))$.
\end{definition}

The significance of $C_t$ is as follows.  Since the algorithm itself is coordinate descent,
and moreover since every line search will be shown to guarantee descent, every candidate
$\lambda$
considered in round $t$ will satisfy $\cL(A\lambda) \leq \cL(A\lambda_{t-1})$;
thus, for every $i\in [m]$,
$\ell(\bfe_i^\top A\lambda) \leq m\cL(A\lambda) \leq m\cL(A\lambda_{t-1})$,
and so $\bfe_i^\top A\lambda \leq \ell^{-1}(m\cL(A\lambda_{t-1}))$, where the inverse is
well-defined since $\ell$ is a bijection between $\R$ and $\R_{++}$ by definition
of $\bL$ (otherwise $C_\ell(z) = \infty$).

The collection of step sizes considered here are as follows, in order of least to most
aggressive.  Throughout these step sizes, $\nu \in (0,1]$ will denote a shrinkage parameter.
\begin{description}
    \item[Quadratic upper bound.]
        Rather than performing an optimal line search,
        i.e., rather than minimizing $\alpha \mapsto \cL(A(\lambda_{t-1} + \alpha v_t))$,
        a quadratic upper bound of this univariate function may be minimized, which has
        a closed form solution (cf. the proof of \Cref{fact:sep:opt:quadub}).
        In particular, define the step size $\alqub{t} := \nu\gamma_t / C_t^4$.
        This choice is pleasant algorithmically only when $C_t$ is easy to compute
        (for instance, $C_t=1$ for the exponential loss).  In general, however, it is
        useful as an analytic aid, since most step sizes here can be lower bounded by it.
        This step size was introduced by \citet[Appendix D.3]{primal_dual_boosting_arxiv}.
    \item[Wolfe.]
        The Wolfe line search is a standard tool from nonlinear optimization
        \citep[chapter 3]{nocedal_wright}, and for convex problems it may be implemented
        with binary search \citep[Appendix D.1]{primal_dual_boosting_arxiv}.
        More precisely, this choice is a set of step sizes $\alwol{t}$ satisfying two conditions.
        First, the step is explicitly disallowed from being too large:
        \begin{align}
            &\cL(A(\lambda_{t-1} + \alpha v_{t}))
            \notag\\
            &\quad\leq \cL(A\lambda_{t-1})
            -\alpha (1-\nu/2) \|A^\top \nabla \cL(A\lambda_{t-1})\|_\infty.
            \label{eq:wolfe:1}
        \end{align}
        Second, the step should be approximately optimal (in terms of the line search problem):
        \begin{align}
            &\nabla \cL(A(\lambda_{t-1} + \alpha v_{t}))^\top A v_{t}
            \notag\\
            &\quad\geq -(1 - \nu/4) \|\nabla\cL(A\lambda_{t-1})^\top A\|_\infty.
            \label{eq:wolfe:2}
        \end{align}
        (Requiring the reverse inequality (with the right hand side negated) yields the
        Strong Wolfe Conditions, which are not necessary here.)
        In contrast to $\alqub{t}$, the Wolfe step does not require knowledge of $C_t$,
        but will yield nearly identical bounds; in fact, computation of the Wolfe step
        requires only function evaluations, gradient evaluations, and knowledge of $\nu$,
        $A$,$v_t$, $\lambda_t$.
    \item[AdaBoost.]
        Following the scheme of AdaBoost, define
        $\alada{t} := \frac \nu 2 \ln(\frac {1+\gamma_t}{1-\gamma_t})$, where
        convention is followed and $\gamma_t=1$ is ignored.
        %set $\alada{t} := \infty$ when $\gamma_t = 1$. %XXX too lazy to analyze
        Unfortunately, even though $\gamma_t$ is loss-dependent, this step will only
        yield rates with the exponential loss.  However, it will be instrumental in analyzing
        the fully optimizing step size, presented next.
        This step size was introduced with the original presentation of AdaBoost
        \citep{freund_schapire_adaboost}, though the analysis here will rather follow
        a slightly later treatment \citep{schapire_singer_confidence_rated}.
    \item[Optimal.]
        Let $\aloptT{t}{1}$
        be a minimizer to $\alpha \mapsto \cL(A(\lambda_{t-1} + \alpha v_t))$,
        which, as in the case of $\alada{t}$, is assumed to exist.
        %or set $\aloptT{t}{1} = \infty$ when no minimizer exists. %XXX LAZY
        For $\nu \in (0,1)$, set $\alopt{t} = \nu \aloptT{t}{1}$.
        When $A$ is binary and $\ell=\exp$, $\alopt{t}= \alada{t}$, though in general
        this is not true.  This step size (with shrinkage!) was suggested by
        \citet{friedman_gradient_boosting} for use with the logistic loss.
\end{description}

To close, note that $\alqub{t}$ and $\alopt{t}$ have a simple relationship.
\begin{proposition}
    \label{fact:singlestep:qub_opt_relation}
    If $A\in [-1,+1]^{m\times n}$ and $\ell\in\bL$, then $\alqub{t} \leq \alopt{t}$.
\end{proposition}

%%XXX meh, tired of citing myself
%%% As a final remark, the presentation here makes the simplifying assumption that a best
%%% hypothesis is chosen in each round.  A workable relaxation is to instead choose a hypothesis
%%% whose advantage over random guessing is within a constant factor of the best possible
%%% \citep[Appendix E]{primal_dual_boosting_arxiv}; such an approach would introduce an extra
%%% scaling term into the margin rates, and this relaxation would have to approach
%%% exact best weak learner selection in order to achieve margin maximization.

\section{The Separable Case}
\label{sec:sep}
This section considers the setting of separability, meaning the weak learning assumption is
satisfied ($\gamma > 0$).  The three subsections respectively provide convergence
rates in empirical risk, basic margin guarantees, and close with some discussion.

\subsection{Convergence of Empirical Risk}
\label{sec:sep:opt}
The basic guarantee is that all of these line search methods, for any loss in $\bL$
and with arbitrary shrinkage,
exhibit the same basic convergence rate as AdaBoost.
\begin{theorem}
    \label{fact:sep:opt:basic}
    Let boosting matrix $A$ with corresponding $\gamma >0$
    and shrinkage parameter $\nu\in(0,1]$ be given.
    Given any $\ell\in \bL$, any $\epsilon>0$, and iterates $\{\lambda_t\}_{t\geq 0}$
    consistent with $\alqub{t}$, $\alwol{t}$,
    $\alopt{t}$,
    or $\alada{t}$ with $\ell = \exp$,
    then
    $\cO(\frac 1 {\gamma^2} \ln(\frac 1 \epsilon))$ iterations suffice
    to ensure $\cL(A\lambda_t) \leq \epsilon$, where the $\cO(\cdot)$ suppresses
    terms depending on $C_1$ and $\nu$.
\end{theorem}

The proof is in the appendix, but a basic discussion will appear here for each step size.
The proofs are straightforward, as they should be: convergence analyses typically prove a
bound for one step, and then iterate the bound.  As such, taking $1/\nu$ steps which are
$\nu$-factor as long as the original should do at least as well as the original (which is
indeed the exhibited trade-off).

First is the quadratic upper bound, which implicitly gives an upper bound for the optimal
step as well.  The proof follows a standard scheme from convex optimization of lower and
upper bounding a potential function based on the gradient; the specifics use the relative
curvature properties of $\bL$, and follow the analysis of
\citet[Section 6.1, Appendix D]{primal_dual_boosting_arxiv}.

\begin{lemma}
    \label{fact:sep:opt:quadub}
    Consider the setting of \Cref{fact:sep:opt:basic},
    but with each step size $\alpha_t$ satisfying
    $\alqub{t} \leq \alpha_t \leq \alopt{t}$.
    %Then for any $t,t_0$ with $t\geq t_0$,
    Then for any $t > t_0 \geq 0$,
    \[
       %\cL(A\lambda_{t+1}) \leq \cL(A\lambda_{t_0})
        \cL(A\lambda_{t}) \leq \cL(A\lambda_{t_0})
        \exp\left(
       %    -\frac {\nu(2-\nu)}{2C_{t_0+1}^6} \sum_{i=t_0+1}^{t+1}\gamma_i^2
            -\frac {\nu(2-\nu)}{2C_{t_0+1}^6} \sum_{i=t_0+1}^{t}\gamma_i^2
        \right).
    \]
\end{lemma}
The reason for the parameter $t_0$ is to mitigate the horrendous dependence on $C_{t_0}$, which
is potentially very large.  In particular, consider $\ell\in \bL_\infty$,
meaning $\lim_{z\to-\infty} C_\ell(z) = 1$.  $C_1$ may be quite bad, but convergence still
happens.  It follows that $C_t \to 1$, and thus, by choosing some large $t_0$, the bound
provides that perhaps there is an initially slow convergence phase, but eventually it is
very fast.  That is to stay, \Cref{fact:sep:opt:quadub} may be applied multiple times to give
a more refined picture of the convergence, particularly in the case that $\ell\in\bL_\infty$,
which guarantees the constants are eventually near 1.

%%  The horrendous dependence on $C_{t_0}$ is mitigated by the fact that the bound may be refined
%%  by ignoring some length-$t_0$
%%  prefix of the iteration sequence which attains some accuracy, and then
%%  re-applying the bound with a refined choice of $C_{t_0}$;
%%  recall that $\ell\in\bL_\infty$ provides that $\lim_{z\to-\infty} C_\ell(z) = 1$.

Next, the Wolfe step size has a similar guarantee (and the analysis once again heavily
relies on techniques due to \citet[6.1, Appendix D]{primal_dual_boosting_arxiv}).

\begin{lemma}
    \label{fact:sep:opt:wolfe}
    Consider the setting of \Cref{fact:sep:opt:basic},
    but with $\alpha_t \in \alwol{t}$.
    %Then for any $t,t_0$ with $t\geq t_0$,
    Then for any $t > t_0 \geq 0$,
    \[
       %\cL(A\lambda_{t+1}) \leq \cL(A\lambda_{t_0})
        \cL(A\lambda_{t}) \leq \cL(A\lambda_{t_0})
        \exp\left(
       %    -\frac {\nu(2-\nu)}{8C_{t_0+1}^6} \sum_{i=t_0+1}^{t+1}\gamma_i^2
            -\frac {\nu(2-\nu)}{8C_{t_0+1}^6} \sum_{i=t_0+1}^{t}\gamma_i^2
        \right).
    \]
\end{lemma}
(The denominator blows up by a factor 4 due to extra halves introduced into the Wolfe conditions,
specifically to adjust around the natural Wolfe parameters being within $(0,1)$ and not $(0,1]$.)

Lastly, consider $\alada{t}$.  As in the statement of \Cref{fact:sep:opt:basic},
this step size is only shown to work with the exponential loss.  This may be an artifact
of the analysis, however, which perhaps follows too closely the treatment
of \citet{schapire_singer_confidence_rated}, which only considers the exponential loss;
for instance, a slightly modified step size can be used to show convergence with
the logistic loss \citep{collins_schapire_singer_adaboost_bregman}.

\begin{lemma}
    \label{fact:sep:opt:ada}
    Consider the setting of \Cref{fact:sep:opt:basic},
    but with $\alpha_t \in \alada{t}$
    %Then for any $t,t_0$ with $t\geq t_0$,
    Then for any $t> t_0 \geq 0$,
    \[
        %\cL(A\lambda_{t+1}) \leq \cL(A\lambda_{t_0})
        \cL(A\lambda_{t}) \leq \cL(A\lambda_{t_0})
        %\prod_{i=t_0+1}^{t+1}
        \prod_{i=t_0+1}^{t}
        C_{i}^3\left(1-\frac\nu 2\gamma_i^2\right).
    \]
\end{lemma}

%%% Note that with the exponential loss ($C_t = 1$) this bound is worse for $\nu \in (0,1)$ than
%%% the quadratic upper bound.

%   The disaster in this bound is the scenario where $C(1-\nu\gamma^2) \geq 1$.
%   As such, $\alopt{t}$ can only be considered a first class citizen for the convergence rates
%   here in the case of the exponential loss, which is hardly a new result (the only distinction,
%   when compared with the rate due to \citet{schapire_singer_confidence_rated}, is the shrinkage
%   term, which intuitively shouldn't matter too much: it forces the algorithm to take many
%   small steps where earlier it'd take some big steps, thus degradation by a factor $1/\nu$ is
%   to be expected).

%   \red{say this mess-up is not surprising due to heavy reliance on the $\exp$ form.}

%   This bound, however, will be useful when studying margins in the next two sections.

\subsection{Margin Maximization}
\label{sec:sep:margins}

The margin rates here follow a simple pattern: the more regularized the step
size, the faster the convergence to a good margin.   While no lower bounds are
presented, this is an interesting and intuitive correspondence (in particular,
consistent with \Cref{fig:MAGIC}).  Unfortunately, the unconstrained step sizes
only have asymptotic convergence (no rates), so the umbrella
\namecref{fact:sep:margins:basic} for this subsection is also asymptotic.

\begin{theorem}
    \label{fact:sep:margins:basic}
    Let boosting matrix $A$ with corresponding $\gamma >0$
    and shrinkage parameter $\nu\in(0,1]$ be given.
    Given any $\ell\in \bL_\infty$, any $\epsilon>0$, and iterates $\{\lambda_t\}_{t\geq 0}$
    consistent with $\alqub{t}$, $\alwol{t}$, $\alada{t}$ with $\ell = \exp$,
    or $\alopt{t}$ with binary $A\in\{-1,+1\}^{m\times n}$,
    then there exists $T$ so that for all $\cM(A\lambda_t) \geq \gamma - \epsilon$
    for all $t\geq T$.
\end{theorem}

In contrast with the convergence rates of empirical risk (e.g.,
\Cref{fact:sep:opt:basic}), the condition $\ell\in\bL_\infty$ is made, rather
than simply $\ell\in\bL$ (with improved constants when $\ell\in\bL_\infty)$.  This can
be interpreted to say: the analysis depends heavily upon the structure of the exponential
loss.  While this condition is likely unnecessary, on the other extreme it is important
for the loss to be strictly convex; if for instance the hinge loss is used, then minimization
can stop at any point achieving zero error, in particular at one with poor margin properties.

%%  \red{\textbf{(OLD TEXT)}
%%  This time, the bounds will differ greatly.  Although some of this may merely
%%  be an analytic artifact (cf. \Cref{sec:sep:discussion}), the causes of the
%%  degradation are intuitive.
%%  }

Returning to task, the quadratic upper bound comes first.

\begin{lemma}
    \label{fact:sep:margin:quadub}
    Suppose the setting of \Cref{fact:sep:margins:basic},
    but with $\alpha_t = \alqub{t}$.
    %Additionally let $t\geq t_0$ be given with
    Additionally let $t> t_0 \geq 0$ be given with
    $t \geq \frac {2 C_1^6 \ln(m)}{\gamma^2\nu(2-\nu)}$
    (whereby all margins are nonnegative by \Cref{fact:sep:opt:quadub}).
    Then
    \begin{align*}
        %\min_{k\in [m]} \frac {-\bfe_k^\top A\lambda_t}{\|\lambda_t\|_1}
        %\cM(A\lambda_{t+1})
        \cM(A\lambda_{t})
        &\geq
        \gamma \left(
            \frac{2-\nu}{2C_{t_0+1}^6}
        \right)
       %- \frac {\ln(c_0)}{(t+1)\nu\gamma},
        - \frac {\ln(c_0)}{t\nu\gamma},
    \end{align*}
    where
    \[
        c_0 \!:= \!
        \max\!\left\{ \!1,
        mC_{t_0+1} \cL(A\lambda_{t_0})
        \exp\left(
            \frac {\nu(2-\nu)}{2C_{t_0+1}^6} \sum_{i=1}^{t_0}\gamma_i^2
        \right) \! \! \right\}\!.
    \]
\end{lemma}

To interpret this bound, first consider the simplifying case that $\ell = \exp$,
whereby $C_t = 1$ for all $t$.  Additionally taking $t_0 = 0$, it follows that
$c_0 = m$, and the bound is simply
\[
    %\cM(A\lambda_{t+1}) \geq \gamma \left(1 - \frac{\nu}{2}\right) - \frac{\ln(m)}{(t+1)\nu\gamma};
    \cM(A\lambda_{t}) \geq \gamma \left(1 - \frac{\nu}{2}\right) - \frac{\ln(m)}{t\nu\gamma};
\]
in particular, $\cM(A\lambda_{t}) \to \gamma$ as $\nu\to 0$ and $t\nu\to\infty$.
For some other $\ell \in \bL_\infty$, the denominator term $C_{t_0+1}^6$ also presents
an obstacle to establishing margin maximization; but note that
$t_0\to\infty$ suffices, since it combines with
$\ell\in\bL_\infty$ via \Cref{fact:sep:opt:basic} to grant $C_{t_0}\to 1$.

The proof of \Cref{fact:sep:margin:quadub}
does not have to work too hard, as the step size appears prominently in
the convergence rate bound (cf. \Cref{fact:sep:opt:quadub}).  As will be discussed
in \Cref{sec:sep:discussion}, the rate is nearly ideal.

The Wolfe search exhibits a similar rate.

\begin{lemma}
    \label{fact:sep:margin:wolfe}
    Suppose the setting of \Cref{fact:sep:margins:basic},
    but with $\alpha_t = \alwol{t}$.
   %Additionally let $t\geq t_0$ be given with
    Additionally let $t> t_0 \geq 0$ be given with
    $t \geq \frac {8 C_1^6 \ln(m)}{\gamma^2\nu(2-\nu)}$
    (whereby all margins are nonnegative by \Cref{fact:sep:opt:wolfe}).
    Then
    \begin{align*}
        %\min_{k\in [m]} \frac {-\bfe_k^\top A\lambda_t}{\|\lambda_t\|_1}
       %\cM(A\lambda_{t+1})
        \cM(A\lambda_{t})
        &\geq
        \gamma \left(
            \frac{2-\nu}{2C_{t_0+1}^2}
        \right)
       %- \frac {4C_1\ln(c_0)}{(t+1)\nu\gamma},
        - \frac {4C_1\ln(c_0)}{t\nu\gamma},
    \end{align*}
    where
    \[
        c_0 \!:= \!
        \max \left\{\!1,
            m C_{t_0+1} \cL(A\lambda_{t_0})\exp\left(
                \frac{(2-\nu)\gamma}{2C_{t_0+1}^2}\sum_{i=1}^{t_0+1}\alpha_i
            \right) \!\!
        \right\}\!.
    \]
\end{lemma}

The preceding two step choices, $\alqub{t}$ and $\alwol{t}$,
had explicit regularization: the first stops as soon
as the steepest matching quadratic turns upward, and the second refuses to go beyond
a boundary (cf. \cref{eq:wolfe:1}).

On the other hand,
the choices $\alada{t}$ and $\alopt{t}$ are only constrained by the data.  Recall that
one way to derive $\alada{t}$ is in the case of binary $A\in\{-1,+1\}^{m\times n}$
and $\ell=\exp$, where it is crucial that each weak learner is wrong on at least
one example: this prevents steps from being too large.  The techniques in the following
proof follow those used in the margin bounds for regular AdaBoost (and are asymptotic there
as well).  It is worth noting that not only is this bound the worst, but the analysis is
the trickiest.

\begin{lemma}
    \label{fact:sep:margins:ada:exploss}
    Consider the setting of \Cref{fact:sep:margins:basic},
    but now $\ell=\exp$
    and $\alpha_t = \alada{t}$.
    Then for any $\epsilon \in (0,\gamma]$,
    there exists $T$
    so that $\cM(A\lambda_t) \geq \gamma - \epsilon$
    for all $t\geq T$.
\end{lemma}

Similarly, $\alopt{t}$ is only implicitly regularized.  The condition that
$A\in\{-1,+1\}^{m\times n}$ prevents the negative, constraining examples
from having too little influence.

\begin{lemma}
    \label{fact:sep:margins:ada:alopt}
    Consider the setting of \Cref{fact:sep:margins:basic},
    but now $\ell=\exp$,
    the matrix $A$ is binary,
    and $\alpha_t = \alopt{t}$.
    Then for any $\epsilon > 0$,
    there exists $T$
    so that $\cM(A\lambda_t) \geq \gamma - \epsilon$
    for all $t\geq T$.
\end{lemma}

The above lemmas together provide the proof of \Cref{fact:sep:margins:basic}.
But before closing, note that while the results for the unconstrained step sizes were
only asymptotic, it is possible to derive a rate for the more modest goal of margins
closer to $\gamma/3$.
%%  , if goal is
%%  Lastly, it is possible to prove a rate of sorts for AdaBoost, albeit to suboptimal margins,
%%  and with no positive attention to shrinkage.

\begin{proposition}
    \label{fact:sep:margins:ada:vague_rate}
    Consider the setting of \Cref{fact:sep:margins:basic}, but specialized
    with $\ell = \exp$ and $\alpha_t = \alada{t}$.
    Let a target margin value $\theta < \gamma$ be given.
    If $\theta < \gamma/(1+\gamma)$ (e.g., it suffices that
    $\theta < \gamma/2$),
    %$\theta < \gamma/2 \leq \gamma / (1+\gamma))$),
    then
    \[
        \frac 1 m\sum_{i=1}^m \1\left[\frac{-\bfe_iA\lambda_t}{\|\lambda_t\|_1} < \theta\right]
        %\leq \exp\left(\frac {t\nu}{2}(-\gamma^2 + \theta\gamma(2+\gamma))\right).
        \leq \exp\left(\!\!\frac {- t\nu(\gamma^2 - \theta\gamma(2+\gamma))}{2}\!\right)\!\!.
    \]
    In particular, if $\theta < \gamma/(2+\gamma)$ (e.g., it suffices that
    %$\theta < \gamma/3 \leq \gamma / (2+\gamma) \leq \gamma/(1+\gamma)$)
    $\theta < \gamma/3$)
    and $t > 2\ln(m) / (\nu(\gamma^2 - \theta\gamma(2+\gamma)))$,
    then $\cM(A\lambda_t) \geq \theta$.
\end{proposition}

Note, of course, that this bound has the severe analytic artifact of demonstrating no
benefit of shrinkage!

%XXX CAN YOU BELIEVE HOW TIRED I WAS:: LELZ
%%  While this bound has a severe analytic artifact (shrinkage strictly hinders),
%%  it is interesting in that not only does it provide a rate, but moreover it is an
%%  average margin bound.
%%    \red{say more, point to
%%  conclusion, also mention reyzin blah}
%%  \blue{WAIT hahahaha wut?  all the bounds here are average margin bounds.. for the love of..}

\subsection{Discussion}
\label{sec:sep:discussion}

To get a sense of these margin bounds, first recall
\citeauthor{yoav_boost_by_majority}'s
lower bound on boosting methods in
the separable case,
which states that $\Omega(\frac 1 {\gamma^2} \ln (\frac 1 \tau))$ iterations are necessary
to achieve classification error $\tau > 0$
\citep[Section 2]{yoav_boost_by_majority}.  Setting $\tau = 1/m$,
it follows that $\Omega(\ln(m)/\gamma^2)$
iterations are necessary to achieve any nonnegative margin.
By comparison, with $\alwol{t}$ and $\ell = \exp$, just $12\ln(m)/\gamma^2$ iterations
with choice $\nu = 1/2$ suffice
to reach margin $\gamma/2$ (by \Cref{fact:sep:margin:wolfe}).
More generally, $\alwol{t}$ reaches margin $\gamma(1-\nu)$ with $8\ln(m)/(\nu\gamma)^2$
iterations (if step size $\alqub{t}$ is used,  then $2\ln(m)/(\nu\gamma)^2$ iterations suffice
by \Cref{fact:sep:margin:quadub}).

The explicit margin-maximizing method of
\citet{shai_singer_weaklearn_linsep} requires $t \geq 32\ln(m)/\epsilon^2$ iterations
to achieve margin $\gamma - \epsilon$, where $\epsilon\in (0,\gamma)$.
By comparison, converting the above multiplicative
bound into an additive bound, step size $\alwolT{t}{\epsilon/\gamma}$ requires
$8\ln(m)/\epsilon^2$ iterations.  While this bound is slightly better, the comparison is not
fair, since $\alwolT{t}{\epsilon/\gamma}$ requires knowledge of $\gamma$ in the choice of
shrinkage parameter $\nu$.
(Pessimistically taking $\nu=\epsilon$ gives an additive guarantee, but with a poor rate.)
Consequently, it can be reasoned that shrinkage methods achieve
excellent margins, but are best suited for multiplicative guarantees.
%XXX I started writing something about running with a decreasing schedule of $nu$.  this is
%subtle because you need to make sure to run each phase long enough.  the bounds don't seem
%to give any advantage to this over restarting with that $\nu$..
%%% \red{where to talk about nonsep?} \blue{also, mention the choice $\nu=\epsilon$, which
%%% gets a worse rate?}

%olds:
%%  \red{
%%  More generally, with $\ell=\exp$ and $\alwol{t}$, setting $\nu := \epsilon/\gamma$
%%  and $t\geq 8\ln(m) / \epsilon^2$ suffices to achieve minimum margin $\gamma - \epsilon$
%%  ($2\ln(m)/\epsilon^2$ with $\alqub{t}$).  The explicit margin-maximizing method
%%  of \citet{shai_singer_weaklearn_linsep} is similar, needing $t\geq 32\ln(m)/\epsilon^2$; note
%%  however that this comparison is unfair, as the choice $\nu = \epsilon/\gamma$ requires
%%  knowledge of $\gamma$, which is not the case with the explict margin-maximizer.
%%  On the other hand, conservatively setting $\nu := \epsilon$ is guaranteed to achieve the
%%  desired margin, but at the expense of the convergence rate.
%%  Lastly,
%%  as these are currently the best margin rates, the choice $\alwol{t}$ is at least theoretically
%%  viable in the sense of margins (please see \citet{shai_singer_weaklearn_linsep} for a
%%  discussion of various methods and bounds).
%%  }

\begin{figure}
\begin{center}
    %\centerline{\includegraphics[width=\columnwidth]{ada_vs_ada}}
    \centerline{\includegraphics[width=\columnwidth]{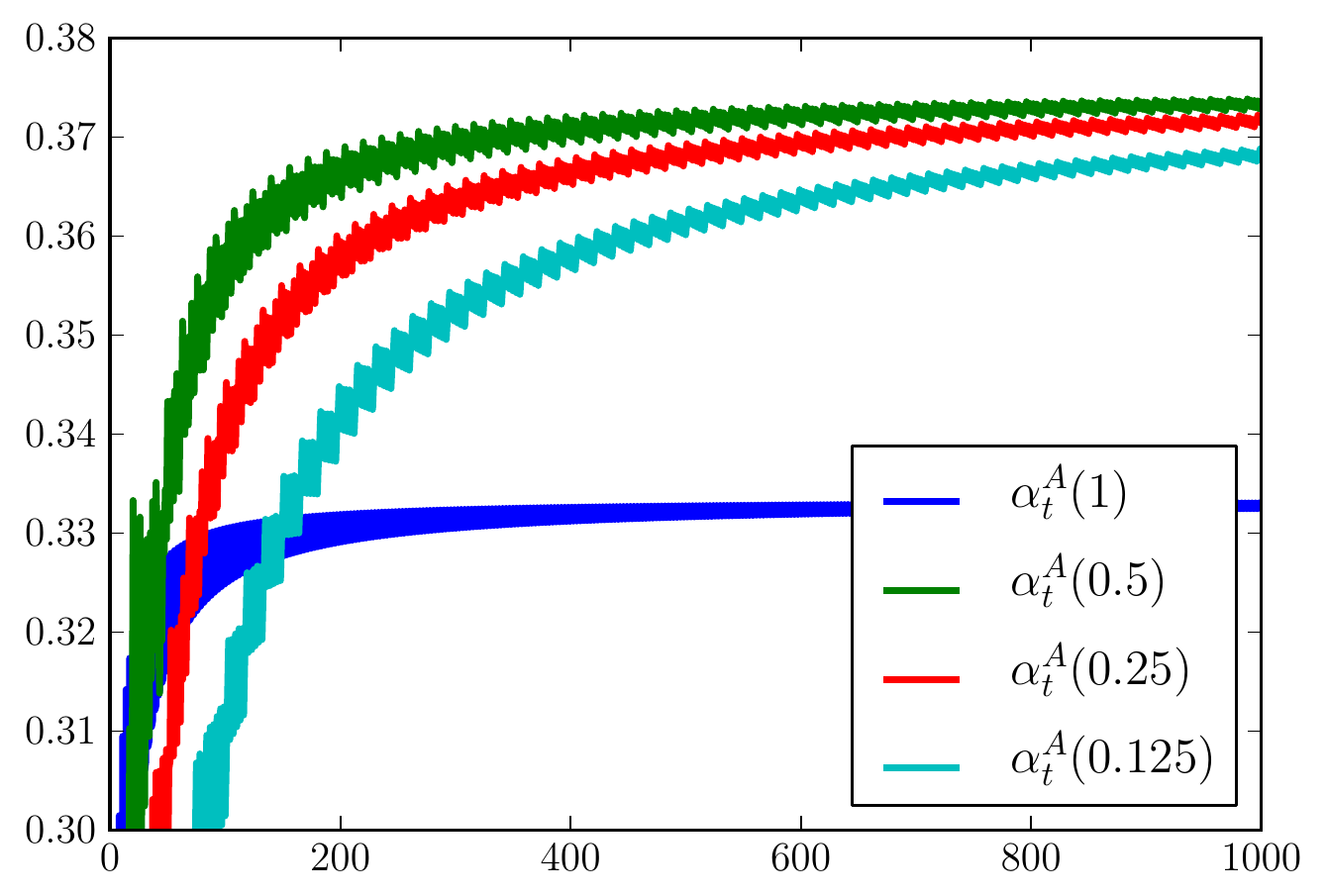}}
    %following to see if pass the icml2013 paper checker.. worked..
    %\centerline{\includegraphics[width=\columnwidth]{ls_contours_hopefully_good_fonts}}
    %\caption{AdaBoost vs. AdaBoost vs. AdaBoost vs. AdaBoost.}
    \caption{Sanity check: shrinkage leads to margin maximization.}
\label{fig:ada_vs_ada}
\end{center}
\end{figure}

Another question is how accurately the bounds presented here depict the methods provided.
As a brief sanity check, the methods may be run on a problem instance where AdaBoost demonstrably
does not achieve maximum margins.
The particular instance tested here is a binary
matrix $A \in \{-1,+1\}^{8\times 8}$ due to
\citet[Theorem 7]{rudin_adaboost_bad_margin};
recall that AdaBoost, in the present notation (with $A$ binary),
corresponds to $\ell = \exp$ and step
size
$\aladaT{t}{1} = \aloptT{t}{1}$ (no shrinkage).
Two plots are provided.
\begin{enumerate}
    \item \Cref{fig:ada_vs_ada} is a sanity check, showing that $\ell = \exp$
        and $\aladaT{t}{1}=\aloptT{t}{1}$
        may not achieve maximum margins, but shrinkage overcomes this.
        %(Since this instance is binary and $\ell = \exp$, then $\alopt{t} =\alada{t}$.)
    \item \Cref{fig:ada_vs_wolfe} demonstrates that the Wolfe search (with $\ell = \exp$)
        is indeed effective, but demanding higher accuracy comes at a price.
\end{enumerate}
These plots will be discussed further in \Cref{sec:discussion}.  Additional tests with
this matrix demonstrated that the method of \citet{shai_singer_weaklearn_linsep} indeed
performs a tiny bit worse than the Wolfe search, but of course one example is not terribly
indicative.  Perhaps most
importantly, a test with the logistic loss showed that the bound is loose: the logistic
loss performs well, and does not suffer a startup cost as indicated by the bounds.

%   \begin{figure}
%   \begin{center}
%       \centerline{\includegraphics[width=\columnwidth]{ada_vs_log}}
%   \caption{nerp}
%   \label{fig:MAGIC}
%   \end{center}
%   \end{figure}

%   \begin{figure}
%   \begin{center}
%       \centerline{\includegraphics[width=\columnwidth]{ada_vs_quadub}}
%   \caption{nerp}
%   \label{fig:MAGIC}
%   \end{center}
%   \end{figure}

%   \begin{figure}
%   \begin{center}
%       \centerline{\includegraphics[width=\columnwidth]{ada_vs_shai}}
%   \caption{nerp}
%   \label{fig:MAGIC}
%   \end{center}
%   \end{figure}

\section{The General Case}
\label{sec:general}

The last technical contribution of this manuscript is to briefly consider the general
case (which is potentially nonseparable).  Similarly to the separable case,
this section will establish convergence rates for empirical risk, margin
guarantees, and briefly discuss the connection to existing margin maximizing
methods.  But first, it is necessary to discuss the structure of the general
case, and in particular to develop what margins mean without separability.

This section hinges upon the following decomposition of a boosting instance.
This decomposition partitions a boosting instance, specifically
its examples $\{(x_i,y_i)\}_{i=1}^m$, into a hard subset $H(A)$,
and an easy subset $H(A)^c$.
The easy subset alone is separable, and thus margins will be measured there.
Although the analysis will rely heavily on properties of this decomposition due to
\citet{primal_dual_boosting_arxiv}, the decomposition itself has appeared, with various
guarantees, in numerous places
\citep{goldreich_levin_hardcore,russell_hardcore,mukherjee_rudin_schapire_adaboost_convergence_rate}.
The notation $H(A)$ reflects the fact that this structure has no relation to the choice
of $\ell\in\bL$.

\begin{figure}[t]
\begin{center}
    %\centerline{\includegraphics[width=\columnwidth]{ada_vs_wolfe}}
    \centerline{\includegraphics[width=\columnwidth]{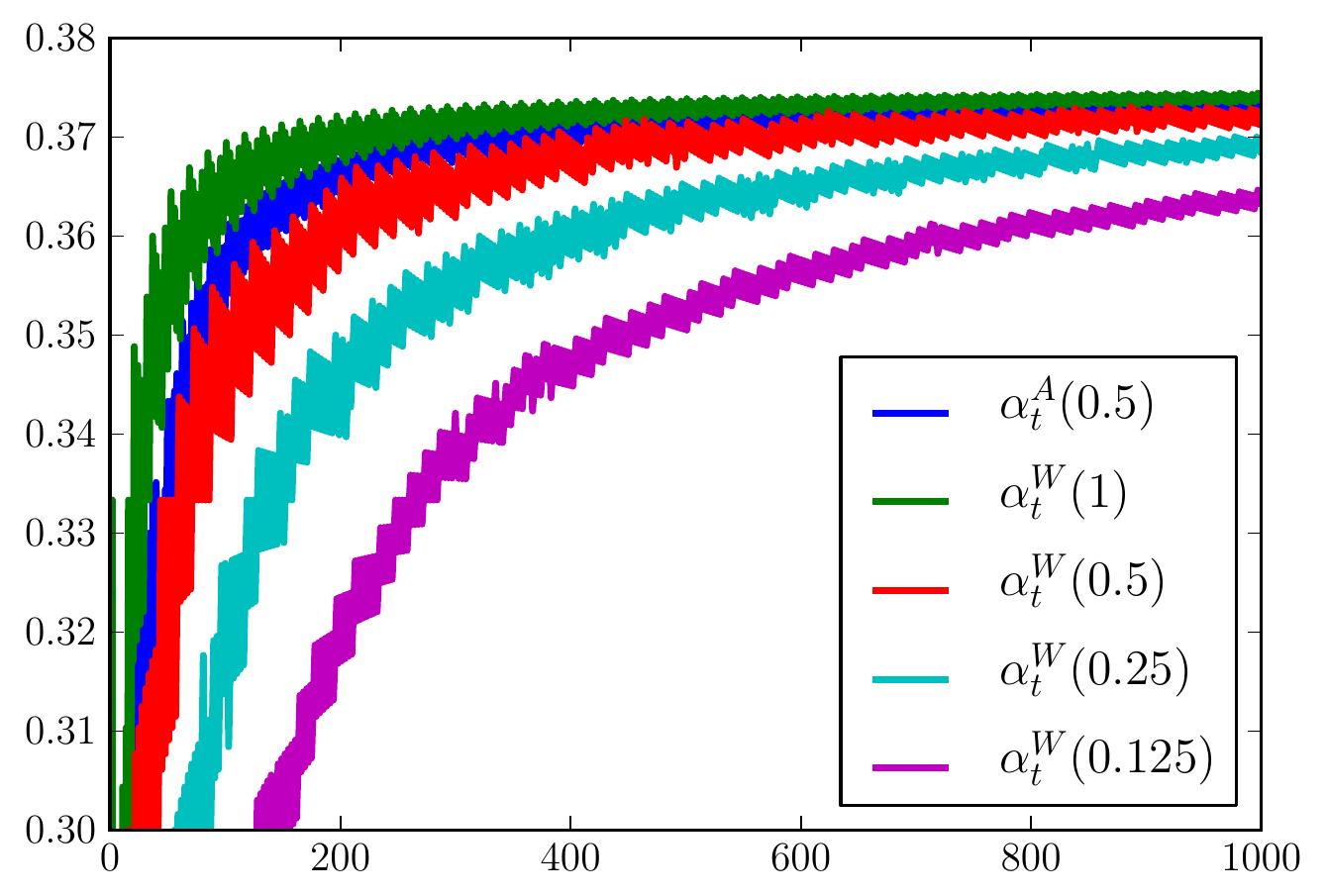}}
    %following to see if pass the icml2013 paper checker.. worked..
    %\centerline{\includegraphics[width=\columnwidth]{ls_contours_hopefully_good_fonts}}
    \caption{Sanity check: the Wolfe search effectively maximizes margins.}
\label{fig:ada_vs_wolfe}
\end{center}
\end{figure}

%olds
%%  \red{
%%  Two guarantees will be provided: convergence rates (of the empirical risk), and a form
%%  of margin guarantee.  Of course, as these instances are not necessarily separable,
%%  the definition of margin must be adjusted.  The approach followed here is to use a
%%  partition of a boosting instance, specifically its examples $\{(x_i,y_i)\}_{i=1}^m$ into
%%  two sets, a hard core $H(A)$ and its complement.  This partition is uniquely determined
%%  by $A$, and has no reliance on the loss function or algorithm.  Although the
%%  tools used here are quite new, this concept is not
%%  \citep{goldreich_levin_hardcore,russell_hardcore,mukherjee_rudin_schapire_adaboost_convergence_rate,primal_dual_boosting_arxiv}.
%%  }

\begin{definition}
    (Cf. \citet[Definition 5.1, 5.7]{primal_dual_boosting_arxiv}.)
    Given a boosting problem encoded in a matrix $A\in\R^{m\times n}$,
    a set of examples (rows) $H(A) \subseteq [m]$
    is a \emph{hard core} for $A$ (and the corresponding boosting problem) if it satisfies
    the following properties.
    \begin{itemize}
        \item There exists a weighting $\hat\lambda\in\R^n$ with $\bfe_i^\top A\hat\lambda < 0$
            for $i\in H(A)^c$ and $\bfe_i^\top A\hat\lambda = 0$ for $i\in H(A)$.
        \item Every weighting $\lambda\in\R^n$ with $\bfe_i^\top A\lambda<0$ for some
            $i\in H(A)$ also has $\bfe_k^\top A\lambda>0$ for some $k\in H(A)$.
    \end{itemize}
    Additionally, define a row-wise partition of $A$ into
    matrices $A_0,A_+$, where $A_+$ has the examples
    in $H(A)$, and $A_0$ has the examples in $H(A)^c$.
\end{definition}

The second property provides that $H(A)$ is difficult: positive margins on some examples force
negative margins on others.  On the other hand, the complement $H(A)^c$ is easy, and moreover
can be solved without affecting $H(A)$.

%crummy attempt:
%%  \red{
%%  The first property of $H(A)$ requires the existence of a weighting $\hat\lambda$ under
%%  which $H(A)^c$ is separated, and $H(A)$ is abstained, meaning $\bfe_i^\top A\hat\lambda$
%%  for $i\in H(A)$.  The second part says that any predictor which is correct on part of $H(A)$
%%  must be incorrect somewhere else on $H(A)$.  The second property may be interepreted as a
%%  certificate of maximality of the first property, and it follows that $H(A)$ is unique.
%%  }

\begin{proposition}(Cf. \citet[Proposition 5.8, Theorem 5.9]{primal_dual_boosting_arxiv}.)
    For any $A\in \R^{m\times n}$,
    a hard core $H(A)$ always exists, and is unique.
\end{proposition}

%%%lelelelelle emphasis on 'the' in following sentence.  originally had a more explicit
%%%statement about 'the' being justified by the above proposition.
With
%\emph{the} hard core
the decomposition
in place, the aforementioned guarantees may be stated.
The first, as in the separable case, is convergence of empirical risk.
There is hardly anything to do here; the groundwork from \Cref{sec:sep} can be
plugged directly into existing techniques to generate this theorem
\citep[Section 6]{primal_dual_boosting_arxiv}.

\begin{theorem}
    \label{fact:general:opt}
    Let general boosting matrix $A$ be given (i.e., potentially $\gamma=0$),
    along with shrinkage parameter $\nu\in (0,1]$, any $\ell \in \bL$,
    and target suboptimality $\epsilon >0$.
    Suppose step sizes $\{\alpha_t\}_{t\geq 0}$ are consistent with
    $\alqub{t}$, $\alwol{t}$, $\alopt{t}$,
    or $\alada{t}$ with $\ell = \exp$ and $A$ binary.
    Then $\cO(\frac 1 \epsilon)$ iterations suffice to reach suboptimality $\epsilon >0$.

    If the instance is either separable (i.e., $\gamma >0$ as in \Cref{sec:sep}) or
    attains its minimizer (i.e., $|H(A)| = m$ \citep[Theorem 5.5]{primal_dual_boosting_arxiv}),
    then the rate improves
    to $\cO(\ln(\frac 1 \epsilon))$.
\end{theorem}

Lastly come the margin guarantees.  As stated above, $H(A)^c$, considered alone, is separable;
note furthermore that the definition of hard core provides the existence of a weighting
$\hat\lambda$.
which has positive margins over $H(A)^c$, but abstains entirely over $H(A)$.
Consequently, an approximate minimizer to $\cL(A\cdot)$ can always add in a scaling
of $\hat\lambda$ and improve its empirical risk while simultaneously improving margins over
$H(A)^c$.  Consequently, it is natural to expect the methods here to achieve positive margins
over $H(A)^c$.  Note that the following result only shows that some positive margins are
attained, and neither assert some sense under which they are maximal, nor
does it provide rates.

\begin{theorem}
    \label{fact:general:margins}
    Let general boosting matrix $A$ be given with $1 \leq |H(A)| \leq m-1$ (i.e.,
    the problem is neither separable, nor is the minimizer attainable).
    Let shrinkage parameter $\nu\in (0,1]$ and any $\ell \in \bL_\infty$ be given.
    Suppose step sizes $\{\alpha_t\}_{t\geq 0}$ are consistent with
    $\alqub{t}$, $\alwol{t}$, $\alopt{t}$ with $\ell=\exp$ and binary $A$,
    or $\alada{t}$ with $\ell = \exp$ and
    binary $A$.,
    Then there exists $\hat\gamma>0$ so that every example off the hard core (i.e.,
    $i\in H(A)^c$) has margin at least $\hat\gamma$ for all large $t$.
\end{theorem}

To close, consider once again the comparison to explicit margin maximizing boosting methods
as presented by \citet{shai_singer_weaklearn_linsep}.
There is no point in discussing the specific method discussed in \Cref{sec:sep:discussion},
whose optimal objective value is exactly $\gamma$, which in this case is zero, and the method
may happily quit without iterating.
Indeed, a primary contribution of \citet{shai_singer_weaklearn_linsep} is not only to address this
issue, but show how the same general boosting scheme can be instantiated for the aforementioned
method, as well as methods with tolerance to nonseparability.

Indeed, consider the ``soft-margin'' boosting method \citep{shai_singer_weaklearn_linsep},
originally due to
\citet{warmuth_soft_margin_boosting}, which, roughly speaking,
has a parameter controlling how many examples to give up on.
This is in contrast to the methods here, which not only have a fixed data-dependant
structure they try less hard on (the hard core $H(A)$), but moreover the particular margins
achieved over the hard core are determined by the loss function $\ell\in\bL$.
It is of course worth mentioning that the margin analysis in the nonseparable case here
is by comparison very incomplete, providing no rates and not even identifying exactly
what positive margins are attained.

%bunch of BS:
%
%%  over the hard core is dependant on 
%%
%%
%%  The presence of a parameter is of course a boon and a bane; it is not pleasant to tune,
%%  but one can imagine settings whe
%%  The methods in the present manuscript have the benefit of not requiring a parameter,
%%  and furthermore not merely giving up on the hard core, but rather preferring some
%%
%%
%%  takes a parameter $k$, and is allowed to give up
%%  \red{\textbf{(THIS IS BS.  GO BACK TO THEIR PAPER.)}} \blue{mention that parameter is a
%%      bane and a boon; bane because hard to set, but boon because, say, can be separable at some
%%  $k$, but $k+1$ gives great margins..}
%%  The primary such scheme, the ``soft-margin'' boosting method (originally due to
%%  \citet{warmuth_soft_margin_boosting}), takes a parameter $k$, and is allowed to give up
%%  on $k$ examples.  The methods considered here do not require a parameter, and moreover are not
%%  only attempting a good margin off the hard core; the fact that they are still performing
%%  empirical risk minimization ensures that some sort of good solution is sought out over $H(A)$,
%%  as determined by the loss function.  On the other hand, the soft-margin boosting methods
%%  can guarantee a good margin on the examples not given up on, whereas
%%  \Cref{fact:general:margins} only provides that some nonzero margin is achieved (but this
%%  may merely be an artifact of the analysis).

\section{Discussion}
\label{sec:discussion}

This manuscript immediately raises a number of questions.  Perhaps foremost is the
general question of the impact of margins on the efficacy of boosting.  Although margins
certainly provide an intuitive theory, it is still unclear how much they directly
correlate with good algorithms~\citep{reyzin_schapire_margins}.

Next, the bounds for the logistic loss are not tight.  As there do not appear to
be any more forgiving analyses of the logistic loss, the natural question is whether
there are new techniques which provide a better characterization.

Lastly, \Cref{fig:ada_vs_ada} shows a threshold effect: shrinkage $1$ does not
lead to the right margin, but $1/2$ and smaller suffices to reach the maximum margin.  (Indeed,
experimentation reveals the threshold to be roughly 0.92.)
It should be possible to clarify this behavior from the perspective of dynamical
systems: smaller steps dodge bad attractors
\citep{rudin_adaboost_bad_margin,rudin_smooth_margin}.

\subsection*{Acknowledgements}
The author thanks Daniel Hsu and the ICML reviewers for helpful comments and
discussions.  The author is also deeply indebted to Robert Schapire for
numerous discussions, insight, and for suggesting study of the unconstrained
step size (at the time, guarantees were only in place for the other choices!).
This work was graciously supported by the NSF under grant IIS-0713540.

\bibliography{ab}

\begin{thebibliography}{25}
\providecommand{\natexlab}[1]{#1}
\providecommand{\url}[1]{\texttt{#1}}
\expandafter\ifx\csname urlstyle\endcsname\relax
  \providecommand{\doi}[1]{doi: #1}\else
  \providecommand{\doi}{doi: \begingroup \urlstyle{rm}\Url}\fi

\bibitem[Bradski(2000)]{opencv_library}
Bradski, G.
\newblock {The OpenCV Library}.
\newblock \emph{Dr. Dobb's Journal of Software Tools}, 2000.

\bibitem[Collins et~al.(2002)Collins, Schapire, and
  Singer]{collins_schapire_singer_adaboost_bregman}
Collins, Michael, Schapire, Robert~E., and Singer, Yoram.
\newblock Logistic regression, {A}da{B}oost and {B}regman distances.
\newblock \emph{Machine Learning}, 48\penalty0 (1-3):\penalty0 253--285, 2002.

\bibitem[Copas(1983)]{copas_shrinkage}
Copas, J.~B.
\newblock Regression, prediction and shrinkage.
\newblock \emph{Journal of the Royal Statistical Society, Series B
  (Methodological)}, 45\penalty0 (3):\penalty0 311--354, 1983.

\bibitem[Freund(1995)]{yoav_boost_by_majority}
Freund, Yoav.
\newblock Boosting a weak learning algorithm by majority.
\newblock \emph{Information and Computation}, 121\penalty0 (2):\penalty0
  256--285, 1995.

\bibitem[Freund \& Schapire(1997)Freund and Schapire]{freund_schapire_adaboost}
Freund, Yoav and Schapire, Robert~E.
\newblock A decision-theoretic generalization of on-line learning and an
  application to boosting.
\newblock \emph{J. Comput. Syst. Sci.}, 55\penalty0 (1):\penalty0 119--139,
  1997.

\bibitem[Friedman(2000)]{friedman_gradient_boosting}
Friedman, Jerome~H.
\newblock Greedy function approximation: A gradient boosting machine.
\newblock \emph{Annals of Statistics}, 29:\penalty0 1189--1232, 2000.

\bibitem[Goldreich \& Levin(1989)Goldreich and Levin]{goldreich_levin_hardcore}
Goldreich, Oded and Levin, Leonid.
\newblock A hard-core predicate for all one-way functions.
\newblock STOC, pp.\  25--32, 1989.

\bibitem[Impagliazzo(1995)]{russell_hardcore}
Impagliazzo, Russell.
\newblock Hard-core distributions for somewhat hard problems.
\newblock In \emph{FOCS}, pp.\  538--545, 1995.

\bibitem[Kearns \& Valiant(1989)Kearns and Valiant]{kearns_valiant_wl}
Kearns, Michael and Valiant, Leslie.
\newblock Cryptographic limitations on learning finite automata and boolean
  formulae.
\newblock STOC, pp.\  433--444, 1989.

\bibitem[Mukherjee et~al.(2011)Mukherjee, Rudin, and
  Schapire]{mukherjee_rudin_schapire_adaboost_convergence_rate}
Mukherjee, Indraneel, Rudin, Cynthia, and Schapire, Robert.
\newblock The convergence rate of {A}da{B}oost.
\newblock In \emph{COLT}, 2011.

\bibitem[Nocedal \& Wright(2006)Nocedal and Wright]{nocedal_wright}
Nocedal, Jorge and Wright, Stephen~J.
\newblock \emph{Numerical optimization}.
\newblock Springer, 2 edition, 2006.

\bibitem[Pedregosa et~al.(2011)Pedregosa, Varoquaux, Gramfort, Michel, Thirion,
  Grisel, Blondel, Prettenhofer, Weiss, Dubourg, Vanderplas, Passos,
  Cournapeau, Brucher, Perrot, and Duchesnay]{scikit-learn}
Pedregosa, F., Varoquaux, G., Gramfort, A., Michel, V., Thirion, B., Grisel,
  O., Blondel, M., Prettenhofer, P., Weiss, R., Dubourg, V., Vanderplas, J.,
  Passos, A., Cournapeau, D., Brucher, M., Perrot, M., and Duchesnay, E.
\newblock {Scikit-learn: Machine Learning in Python }.
\newblock \emph{Journal of Machine Learning Research}, 12:\penalty0 2825--2830,
  2011.

\bibitem[R\"atsch et~al.(2001)R\"atsch, Onoda, and
  M\"uller]{ratsch_soft_margins_adaboost}
R\"atsch, G., Onoda, T., and M\"uller, K.-R.
\newblock Soft margins for adaboost.
\newblock \emph{Machine Learning}, 42:\penalty0 287--320, 2001.

\bibitem[R\"atsch \& Warmuth(2005)R\"atsch and
  Warmuth]{warmuth_maxmargin_boosting}
R\"atsch, Gunnar and Warmuth, Manfred.
\newblock Efficient margin maximizing with boosting.
\newblock \emph{Journal of Machine Learning Research}, 6:\penalty0 2153--2175,
  2005.

\bibitem[Reyzin \& Schapire(2006)Reyzin and Schapire]{reyzin_schapire_margins}
Reyzin, Lev and Schapire, Robert~E.
\newblock How boosting the margin can also boost classifier complexity.
\newblock In \emph{In Proceedings of the 23rd International Conference on
  Machine Learning}, pp.\  753--760, 2006.

\bibitem[Rudin et~al.(2004)Rudin, Daubechies, and
  Schapire]{rudin_adaboost_bad_margin}
Rudin, Cynthia, Daubechies, Ingrid, and Schapire, Robert~E.
\newblock The dynamics of {A}da{B}oost: cyclic behavior and convergence of
  margins.
\newblock \emph{Journal of Machine Learning Research}, 5:\penalty0 1557--1595,
  2004.

\bibitem[Rudin et~al.(2007)Rudin, Schapire, and
  Daubechies]{rudin_smooth_margin}
Rudin, Cynthia, Schapire, Robert~E., and Daubechies, Ingrid.
\newblock Analysis of boosting algorithms using the smooth margin function.
\newblock \emph{Annals of Statistics}, 35\penalty0 (6):\penalty0 2723--2768,
  2007.

\bibitem[Schapire \& Freund(2012)Schapire and
  Freund]{schapire_freund_book_final}
Schapire, Robert~E. and Freund, Yoav.
\newblock \emph{Boosting: Foundations and Algorithms}.
\newblock MIT Press, 2012.

\bibitem[Schapire \& Singer(1999)Schapire and
  Singer]{schapire_singer_confidence_rated}
Schapire, Robert~E. and Singer, Yoram.
\newblock Improved boosting algorithms using confidence-rated predictions.
\newblock \emph{Machine Learning}, 37\penalty0 (3):\penalty0 297--336, 1999.

\bibitem[Schapire et~al.(1997)Schapire, Freund, Barlett, and
  Lee]{boosting_margin}
Schapire, Robert~E., Freund, Yoav, Barlett, Peter, and Lee, Wee~Sun.
\newblock Boosting the margin: A new explanation for the effectiveness of
  voting methods.
\newblock In \emph{ICML}, pp.\  322--330, 1997.

\bibitem[Shalev-Shwartz \& Singer(2008)Shalev-Shwartz and
  Singer]{shai_singer_weaklearn_linsep}
Shalev-Shwartz, Shai and Singer, Yoram.
\newblock On the equivalence of weak learnability and linear separability: New
  relaxations and efficient boosting algorithms.
\newblock In \emph{COLT}, pp.\  311--322, 2008.

\bibitem[Steele(2004)]{c-s_master}
Steele, J.~Michael.
\newblock \emph{The {C}auchy-{S}chwarz Master Class}.
\newblock Cambridge University Press, 2004.

\bibitem[Telgarsky(2012)]{primal_dual_boosting_arxiv}
Telgarsky, Matus.
\newblock A primal-dual convergence analysis of boosting.
\newblock 2012.
\newblock {\tt arXiv:1101.4752v3 [cs.LG]}.

\bibitem[Warmuth et~al.(2006)Warmuth, Liao, and
  R{\"a}tsch]{warmuth_soft_margin_boosting}
Warmuth, Manfred~K., Liao, Jun, and R{\"a}tsch, Gunnar.
\newblock Totally corrective boosting algorithms that maximize the margin.
\newblock In \emph{ICML}, pp.\  1001--1008, 2006.

\bibitem[Zhang \& Yu(2005)Zhang and Yu]{zhang_yu_boosting}
Zhang, Tong and Yu, Bin.
\newblock Boosting with early stopping: Convergence and consistency.
\newblock \emph{The Annals of Statistics}, 33:\penalty0 1538--1579, 2005.

\end{thebibliography}
\bibliographystyle{icml2013}

%\end{document}

\clearpage
\appendix
\section{Deferred Material from \Cref{sec:setup}}
\begin{proof}[Proof of \Cref{fact:exp_log_bL}]
    There is nothing to show for $\exp$, so consider $\ell(x) = \ln(1+\exp(x))$,
    let $z\in \R$ be given, and let $x\leq z$ be arbitrary.

    Concavity grants $\ln(1+\exp(x))\leq \exp(x)$.
    The lower bound can be checked in two stages.  First, if
    $x\leq \min\{-1,z\}$, a Taylor expansion gives
    \begin{align*}
        \ln(1+e^x)
        &\geq e^x - \sup_{\xi\in\R} \frac {1}{2(1+\xi)^2} e^{2x}
        \\
    &\geq e^x\left(1 - \frac{\min\{e^z,e^{-1}\}}{2}\right).
    \end{align*}
    On the other hand, if $-1 \leq x \leq z$, then
    $e^x \leq e^z \ln(1+e^{-1}) / \ln(1+e^{-1})
    \leq e^z \ln(1+e^{x}) / \ln(1+e^{-1})$.
   %On the other hand, for $x\geq -1$, the ratio
   %$e^x / \ln(1+e^x)$ is increasing (by calculus),
   %thus $e^x \leq  e^z \ln(1+e^x)$.

    Next, $\ell'(x) = e^x / (1+e^x)$, so $\ell'(x) \leq e^x \leq \ell'(x)(1+e^z)$.
    Similarly, $\ell''(x) = e^x / (1+e^x)^2$, so
    $\ell''(x) \leq e^x \leq \ell''(x)(1+e^z)^2$.
\end{proof}

The following \namecref{fact:qub:singlestep} (and its proof)
derive $\alqub{t}$, establishes $\alqub{t} \leq \alopt{t}$,
and gives the basic improvement due to one step satisfying $\alpha \in [\alqub{t},\alopt{t}]$.

\begin{lemma}
    \label{fact:qub:singlestep}
    Let boosting matrix $A$, shrinkage parameter $\nu\in(0,1]$, and any $\ell\in\bL$ be
    given.  For any iteration $t$, it holds that $\alqub{t+1} \leq \alopt{t+1}$.
    Furthermore, any step $\alpha \in [\alqub{t+1},\alopt{t+1}]$ satisfies
    \[
        \cL(A(\lambda_t + \alpha v_{t+1}))
        \leq
        \cL(A\lambda_t)\exp\left(
            - \frac {\nu(2-\nu)\gamma_{t+1}^2}{2C_{t+1}^6}
        \right).
    \]
\end{lemma}
\begin{proof}
    This analysis follows a scheme laid out by
    \citet[Appendix D.3]{primal_dual_boosting_arxiv}.
    Let $t$ denote any fixed iteration,
    and $I$ denote the (possibly unbounded) interval
    \[
        I := \left\{
            \alpha \geq 0
            :
            \cL(A(\lambda_t + \alpha v_{t+1}))
            \leq \cL(A\lambda_t)
        \right\};
    \]
    by continuity of $\cL$ and choice of $v_{t+1}$, $I$ is nonempty, with nonempty
    interior.
    By second order Taylor expansion,
    every $\alpha\in I$ satisfies
    \begin{align*}
        &\cL(A(\lambda_t + \alpha v_{t+1}))
        \\
        &\leq
        \cL(A\lambda_t) + \alpha v_{t+1}^\top A^\top \nabla\cL(A\lambda_t)
        \\&
        \quad+ \sup_{r \in I}
        \frac {\alpha^2}{2} v_{t+1}^\top A^\top \nabla ^2 \cL(A(\lambda_t +r v_{t+1})) A v_{t+1}
        \\
        &\leq
        \cL(A\lambda_t) - \alpha \|A^\top \nabla\cL(A\lambda_t)\|_\infty
        \\&
        \quad
        +
        \frac {\alpha^2}{2}
        \sup_{r \in I}
        \frac 1 m
        \sum_{i=1}^m \ell''(\bfe_i^\top A(\lambda_t + r v_{t+1})) A_{ij_{t+1}}^2
        \\
        &\leq
        \cL(A\lambda_t) - \alpha \|A^\top \nabla\cL(A\lambda_t)\|_\infty
        \\&
        \quad
        +
        \frac {C_{t+1}^2\alpha^2}{2}
        \sup_{r \in I}
        \frac 1 m
        \sum_{i=1}^m \ell(\bfe_i^\top A(\lambda_t + r v_{t+1})) A_{ij_{t+1}}^2
        \\
        &=
        \cL(A\lambda_t) - \alpha \|A^\top \nabla\cL(A\lambda_t)\|_\infty
        +
        \frac {C_{t+1}^2\alpha^2}{2}
        \cL(A\lambda_t)
        \\
        &\leq
        \cL(A\lambda_t) - \alpha \|A^\top \nabla\cL(A\lambda_t)\|_\infty
        +
        \frac {C_{t+1}^4\alpha^2}{2}
        \|\nabla\cL(A\lambda_t)\|_1,
    \end{align*}
    which made use of $\ell'' \leq C_{t+1}\exp \leq C_{t+1}^2\ell$ along $I$,
    $\ell \leq C_{t+1}\exp \leq C_{t+1}^2\ell'$ along $I$,
    $|A_{ij}|\leq 1$ (since
    elements of $\cH$ are bounded in this way),
    and the definition of $I$ (specifically $r=0$ is the worst choice for $r\in I$).
    This final expression is a quadratic, whose
    minimizer must lie within $I$ (since its second derivative exceeds that of
    $\cL$ along this interval).  Differentiating and setting to zero, the
    minimizer is
    \[
        I
        \ni
        \frac {\|A^\top \nabla\cL(A\lambda_t)\|_\infty}
        {C_{t+1}^4\|\nabla \cL(A\lambda_t)\|_1}
        = \frac {\gamma_{t+1}}{C_{t+1}^4}
        %only for exploss:
        %= \frac {\|A^\top \nabla\cL(A\lambda_t)\|_\infty}{\|\nabla\cL(A\lambda_t)\|_1}
        %= \eta\gamma_{t+1}
        = \alqubT{t+1}{1}.
    \]
    This provides a derivation of the step $\alqub{t+1}$, and also shows
    $\alqubT{t+1}{1} \leq \aloptT{t+1}{1}$.
    Plugging $\alqub{t+1}$ in for $\alpha$ in the above quadratic upper bound,
    \begin{align*}
        \cL(A(\lambda_t + \alpha v_{t+1}))
        &\leq
        \cL(A\lambda_t)
        - \frac {\nu(2-\nu)\gamma_{t+1}^2\|\nabla\cL(A\lambda_t)\|_1}{2C_{t+1}^4}
        \\
        &\leq
        \cL(A\lambda_t)\left(
            1 - \frac {\nu(2-\nu)\gamma_{t+1}^2}{2C_{t+1}^6}
        \right)
        \\
        &\leq
        \cL(A\lambda_t)\exp\left(
            - \frac {\nu(2-\nu)\gamma_{t+1}^2}{2C_{t+1}^6}
        \right).
        %\qedhere % nope, since causes VERY overfull hbox
    \end{align*}
\end{proof}

\begin{proof}[Proof of \Cref{fact:singlestep:qub_opt_relation}]
    This is the first part of \Cref{fact:qub:singlestep}.
\end{proof}

\section{Deferred Material from \Cref{sec:sep}}
\subsection{Deferred Material from \Cref{sec:sep:opt}}
\begin{proof}[Proof of \Cref{fact:sep:opt:quadub}]
    By \Cref{fact:qub:singlestep}, for any $t$,
    \begin{align*}
        \cL(A(\lambda_t + \alpha v_{t+1}))
        &\leq
        \cL(A\lambda_t)\exp\left(
            - \frac {\nu(2-\nu)\gamma_{t+1}^2}{2C_{t+1}^6}
        \right).
    \end{align*}
    Now let $t_0 \leq t$ be given as in the desired statement,
    apply this bound $t-t_0$ times,
    and use the fact that $C_{t+1} \leq C_t$.
    %XXX \red{(do I need to separate out and prove this last thing somewhere)}.
    %technically not fully trivial: requires all methods to be descent methods..
\end{proof}

\begin{proof}[Proof of \Cref{fact:sep:opt:wolfe}]
    Let $t$ denote any fixed iteration.
    Substituting $c_1 = 1 - \nu/2$, $c_2 = 1-\nu/4$, and $\eta = C_{t+1}^2$
    in a nearly identical
    guarantee for the Wolfe line search \citep[Proposition D.6]{primal_dual_boosting_arxiv}
    (where $\eta$ is simply the biggest ratio between $\ell$ and $\ell''$ in the current
    sublevel set)
    provides
    \begin{align*}
        &\cL(A(\lambda_t + \alpha v_{t+1}))
        \\
        &\leq
        \cL(A\lambda_t)
        - \frac {(1-\nu/2)(\nu/4)\|A^\top \nabla \cL(A\lambda_t)\|_\infty^2}{2C_{t+1}^2\cL(A\lambda_t)}
        \\
        &\leq
        \cL(A\lambda_t)\left(
            1 - \frac {\nu(2-\nu)\gamma_{t+1}^2}{8C_{t+1}^6}
        \right)
        \\
        &\leq
        \cL(A\lambda_t)\exp\left(
            - \frac {\nu(2-\nu)\gamma_{t+1}^2}{8C_{t+1}^6}
        \right).
    \end{align*}
    Given $t_0\leq t$, applying this bound $t-t_0$ times and using $C_{t+1} \leq C_t$
    gives the result.
\end{proof}

Next, instead of directly proving \Cref{fact:sep:opt:ada},
a more general \namecref{fact:sep:alada:opt} is given first, which
will be useful later.

\begin{lemma}
    \label{fact:sep:alada:opt}
    Consider the setting of \Cref{fact:sep:opt:basic},
    except now each step size $\alpha_i$ satisfies
    \[
        \alada{i}-\tau \leq \alpha_i \leq \alada{i}+\tau
    \]
    for some $\tau > 0$.  Then, given $t\geq t_0$,
    \begin{align*}
        &\cL(A\lambda_{t+1})
        \\&\leq
        \cL(A\lambda_{t_0})
        \\
        &\quad\cdot
        \prod_{i=t_0+1}^{t+1}
        \frac {e^\tau C_{i}^4}{2} (1-\gamma_i^2)^{\nu/2}((1+\gamma_i)^{1-\nu} + (1-\gamma_i)^{1-\nu})
        \\&\leq
        \cL(A\lambda_{t_0})
        \prod_{i=t_0+1}^{t+1}
        \frac {e^\tau C_{i}^4}{2} (1-\gamma_i^2)^{\nu/2}.
    \end{align*}
\end{lemma}
\begin{proof}
    Fix an iteration $t$, and set
    $w_i = \ell'(\bfe_i^\top A\lambda_t)$ and $W = \sum_i w_i \leq mC_{t+1}^2 \cL(A\lambda_t)$.
    By convexity of $\exp(\cdot)$,
    \begin{align*}
        &\cL(A\lambda_{t+1})
        \\
        &\leq
        \frac {C_{t+1}}{m} \sum_{i=1}^m \exp(\bfe_i^\top A\lambda_{t+1})
        \\
        &\leq
        \frac{C_{t+1}^2}{m}
        \left(
            \frac W W
        \right)
        \sum_{i=1}^m w_i \exp\Bigg(
            \frac{1+\bfe_i^\top Av_{t+1}}{2}\alpha_{t+1}
            \\
            &\qquad\qquad
            +\frac{1-\bfe_i^\top Av_{t+1}}{2}(-\alpha_{t+1})
            \Bigg)
        \\
        &\leq
        \frac {C_{t+1}^2}{m}
        W
        \left(
            \frac{1-\gamma_{t+1}}2 \exp(\alpha_{t+1})
            + \frac{1+\gamma_{t+1}}2 \exp(-\alpha_{t+1})
        \right)
        \\
        &\leq
        \frac{e^\tau C_{t+1}^4}{2}
        \cL(A\lambda_t)
        (1-\gamma_{t+1}^2)^{\nu/2}
        \\
        &\quad\cdot
        \left(
            (1-\gamma_{t+1})^{1-\nu}
            + (1+\gamma_{t+1})^{1-\nu}
        \right).
    \end{align*}
    To simplify this expression, note that
    $(\cdot)^{1-\nu}$ is a concave function, and thus
    \begin{align*}
        &\frac {(1-\gamma_{t+1})^{1-\nu}} 2 + \frac {(1+\gamma_{t+1})^{1-\nu}} 2
        \\
        &\leq \left(
            \frac {1-\gamma_{t+1}} 2 + \frac {1+\gamma_{t+1}} 2
        \right)^{1-\nu}
        \\
        &= 1.
    \end{align*}
    To finish, given $t \geq t_0$, the result follows by $t-t_0$ applications
    of these bounds. %, and noting $C_{i+1} \geq C_i$.
\end{proof}

\begin{proof}[Proof of \Cref{fact:sep:opt:ada}]
    This follows by taking the second bound in
    \Cref{fact:sep:alada:opt} with the choice $\tau = 0$.
\end{proof}

\begin{proof}[Proof of \Cref{fact:sep:opt:basic}]
    The result follows from
    \Cref{fact:sep:opt:quadub},
    \Cref{fact:sep:opt:wolfe},
    and
    \Cref{fact:sep:opt:ada}
    with the choice $t_0 = 0$
    and using $C_{t+1} \leq C_t$.
\end{proof}

\subsection{Deferred Material from \Cref{sec:sep:margins}}

\begin{proof}[Proof of \Cref{fact:sep:margin:quadub}]
    To start, note that
    \begin{align*}
     %% \|\lambda_{t_0}\|_1
     %% &= \|\sum_{i=1}^{t_0} \nu \gamma_i v_i\|_1
     %%%\\
     %%%&
     %% \leq \nu \sum_{i=1} \gamma_i,
     %% \\
     %% \|\lambda_t\|_1
     %% &= \|\lambda_{t_0} + \nu\sum_{i=t_0}^t v_i\gamma_i\|_1
     %% %\\
     %% %&
     %% \leq \|\lambda_{t_0}\|_1 + \nu\sum_{i=t_0}^t v_i\gamma_i.
        \|\lambda_{t+1}\|_1
        &= \left\|\nu\sum_{i=1}^{t+1} v_i\gamma_iC_{i}^{-4}\right\|_1
        %\\
        %&
        \leq \nu C_1^{-4}\sum_{i=1}^{t+1} \gamma_i
        \leq \nu\sum_{i=1}^{t+1} \gamma_i.
    \end{align*}
    By the form of $\cL$ and the optimization guarantee in
    \Cref{fact:sep:opt:quadub},
    \begin{align*}
        &\max_{k\in[m]} \exp(\bfe_k^\top A\lambda_{t+1})
        \\
        &\leq
        \sum_{i=1}^m \exp(\bfe_i^\top A\lambda_{t+1})
        \\
        &\leq mC_{t+1} \cL(A\lambda_{t+1})
        \\
        &\leq mC_{t_0+1} \cL(A\lambda_{t_0})\exp\left(
            - \frac {\nu(2-\nu)}{2C_{t_0+1}^6} \sum_{i=t_0+1}^{t+1}\gamma_i^2
        \right)
        \\
        &= mC_{t_0+1} \cL(A\lambda_{t_0})
        \exp\left(
            \frac {\nu(2-\nu)}{2C_{t_0+1}^6} \sum_{i=1}^{t_0}\gamma_i^2
        \right)
        \\
        &\quad\cdot
        \exp\left(
            - \frac {\nu(2-\nu)}{2C_{t_0+1}^6} \sum_{i=1}^{t+1}\gamma_i^2
        \right)
        \\
        &\leq
        c_0
        \exp\left(
            - \frac {(2-\nu)}{2C_{t_0+1}^6} \sum_{i=1}^t\nu\gamma_i^2
        \right),
    \end{align*}
    where $c_0$ is as in the statement.  Since $\ln(\cdot)$ is increasing,
    it follows that
    \[
        \max_{k\in[m]} \bfe_k^\top A\lambda_{t+1}
        \leq
        - \frac {(2-\nu)}{2C_{t_0+1}^6} \sum_{i=1}^t\nu\gamma_i^2
        +\ln(c_0).
    \]
  %%For convenience, set
  %%\[
  %%    c_0 :=
  %%    \max\left\{1,
  %%    mC_t \cL(A\lambda_{t_0})
  %%    \exp\left(
  %%        \frac {\nu(2-\nu)}{2C_t^6} \sum_{i=0}^{t_0}\gamma_i^2
  %%    \right)\right\}.
  %%\]
    Using the above
    bound on $\|\lambda_t\|_1$, since $t\gamma \leq \sum_{i=1}^{t+1} \gamma_t$,
    and $-\bfe_k^\top A\lambda_{t+1}$ is nonnegative by the lower bound on $t$,
    \begin{align*}
        \min_{k\in [m]} \frac {-\bfe_k^\top A\lambda_{t+1}}{\|\lambda_{t+1}\|_1}
        &\geq
        \min_{k\in [m]} \frac {-\bfe_k^\top A\lambda_{t+1}}{\nu\sum_{i=1}^{t+1}\gamma_i}
        \\
        &\geq
        \gamma \left(
            \frac{2-\nu}{2C_{t_0+1}^6}
        \right)
        - \frac {\ln(c_0)}{\nu\sum_{i=1}^{t+1} \gamma_i}
        \\
        &\geq
        \gamma \left(
            \frac{2-\nu}{2C_{t_0+1}^6}
        \right)
        - \frac {\ln(c_0)}{(t+1)\nu\gamma}.
        \qedhere
    \end{align*}
\end{proof}

\begin{proof}[Proof of \Cref{fact:sep:margin:wolfe}]
    Any step size $\alpha_{t+1}$ satisfying the Wolfe conditions
    will have lower bound
    \begin{align*}
        \alpha_{t+1}
        &\geq \frac {(1-(1-\nu/4))\|A^\top\cL(A\lambda_t)\|_\infty}{C_t^2\cL(A\lambda_t)}
        \\
        &\geq \frac {(1-(1-\nu/4))\|A^\top\cL(A\lambda_t)\|_\infty}{C_t^4\|\nabla\cL(A\lambda_t)\|_1}
        \\
        &= \frac {\nu\gamma_{t+1}}{4C_{t+1}^4}
        \geq \frac {\nu\gamma}{4C_{1}^4};
    \end{align*}
    indeed this expression appears in proofs demonstrating the improvement due to a single
    step of the Wolfe search, see for instance \citet[Proof of Proposition D.6, second
    to last line]{primal_dual_boosting_arxiv}.

    Additionally, note
    \[
        \|\lambda_{t+1}\|_1
        = \left\|\sum_{i=1}^{t+1} \alpha_iv_i\right\|_1
        \leq \sum_{i=1}^{t+1}\alpha_i.
    \]

    Direct from the first Wolfe condition (\cref{eq:wolfe:1}),
    \begin{align*}
        &\cL(A\lambda_{t+1})
        \\
        &=
        \cL(A(\lambda_t + \alpha_{t+1} v_{t+1}))
        \\
        &\leq \cL(A\lambda_t) - \alpha_{t+1} (1-\nu/2) \|A^\top \nabla \cL(A\lambda_t)\|_\infty
        \\
        &\leq \cL(A\lambda_t)
        \left(1 -
        \frac{\alpha_{t+1} (1-\nu/2) \|A^\top \nabla \cL(A\lambda_t)\|_\infty}{\cL(A\lambda_t)}\right)
        \\
        &\leq \cL(A\lambda_t)
        \left(1 -
        \frac{\alpha_{t+1} (2-\nu)\gamma_{t+1}}{2C_{t+1}^2}\right).
    \end{align*}
    Now let $t\geq t_0$ be given as in the statement.
    Applying the above inequality $t-t_0$ times,
    \begin{align*}
        &\max_{k\in[m]} \exp(\bfe_{k}^\top A\lambda_{t+1})
        \\
        &\leq m C_{t+1} \cL(A\lambda_{t+1})
        \\
        &\leq m C_{t_0+1} \cL(A\lambda_{t_0})
        \exp\left(
            -\frac{(2-\nu)\gamma}{2C_{t_0+1}^2}\sum_{i=t_0+1}^{t+1} \alpha_i
        \right)
        \\
        &\leq m C_{t_0+1} \cL(A\lambda_{t_0})\exp\left(
        \frac{(2-\nu)\gamma}{2C_{t_0+1}^2}\sum_{i=1}^{t_0}\alpha_i
        \right)
        \\
        &\quad\cdot\exp\left(
        -\frac{(2-\nu)\gamma}{2C_{t_0+1}^2}\sum_{i=1}^{t+1} \alpha_i
        \right)
        \\
        &\leq c_0
        \exp\left(
            -\frac{(2-\nu)\gamma}{2C_{t_0+1}^2}\sum_{i=1}^{t+1} \alpha_i
        \right),
    \end{align*}
    where $c_0$ is as in the statement.  Since $\ln(\cdot)$ is increasing,
    it follows
    that
    \[
        \max_{k\in[m]} \bfe_{k}^\top A\lambda_{t+1}
        \leq
            -\frac{(2-\nu)\gamma}{2C_{t_0+1}^2}\sum_{i=1}^{t+1} \alpha_i
            + \ln(c_0).
    \]
    Using the above lower bound on $\alpha_i$ in
    terms of $\gamma_i$, and since all margins are nonnegative by the lower bound on $t$,
    \begin{align*}
        \min_{k\in [m]} \frac {-\bfe_k^\top A\lambda_{t+1}}{\|\lambda_{t+1}\|_1}
        &\geq
        \min_{k\in [m]} \frac {-\bfe_k^\top A\lambda_{t+1}}{\sum_{i=1}^{t+1}\alpha_i}
        \\
        &\geq
        \gamma \left(
            \frac{2-\nu}{2C_{t_0+1}^2}
        \right)
        - \frac {\ln(c_0)}{\sum_{i=1}^{t+1} \alpha_i}
        \\
        &\geq
        \gamma \left(
            \frac{2-\nu}{2C_{t_0+1}^2}
        \right)
        - \frac {4C_1^4\ln(c_0)}{(t+1)\nu\gamma}.
        \qedhere
    \end{align*}
\end{proof}

The remainder of this subsection provides proofs for $\alada{t}$ and $\alopt{t}$,
but uses some later material, most specifically the quantity $\Upsilon_\nu$.

\begin{lemma}
    \label{fact:sep:margins:ada:step1}
    Consider the setting of \Cref{fact:sep:opt:basic},
    except now each step size $\alpha_i$ satisfies
    \[
        \alada{i}-\tau \leq \alpha_i \leq \alada{i}+\tau
    \]
    for some $\tau > 0$.
    Let $\theta \in [0,\gamma)$ be given.
    Then, given $t\geq t_0$,
    \begin{align*}
        &\sum_{i=1}^m \1\left[\frac{-\bfe_i A\lambda_{t+1}}{\|\lambda_{t+1}\|_1} < \theta \right]
        \\
        &\leq mC_{t+1} \exp(\theta\|\lambda_{t_0}\|_1)\cL(A\lambda_{t_0})
        \prod_{i=t_0+1}^{t+1}\Bigg(e^{\theta\tau}
        \left(
            \frac {1+\gamma_i}{1-\gamma_i}
        \right)^{\theta\nu/2}
        \\
        &\quad\cdot\frac {e^\tau C_{i}^4}{2} (1-\gamma_i^2)^{\nu/2}((1+\gamma_i)^{1-\nu} + (1-\gamma_i)^{1-\nu})
        \Bigg).
        \\
        &\leq mC_{t+1} \exp(\theta\|\lambda_{t_0}\|_1)\cL(A\lambda_{t_0})
        \prod_{i=t_0+1}^{t+1}\Bigg(e^{\theta\tau}
        \left(
            \frac {1+\gamma_i}{1-\gamma_i}
        \right)^{\theta\nu/2}
        \\
        &\quad\cdot e^\tau C_{i}^4 (1-\gamma_i^2)^{\nu/2}
        \Bigg).
        \qedhere
    \end{align*}
\end{lemma}
\begin{proof}
    To start,
    \begin{align*}
        &\sum_{i=1}^m
        \1\left[\frac{-\bfe_i A\lambda_{t+1}}{\|\lambda_{t+1}\|_1} < \theta \right]
        \\
        &=
        \sum_{i=1}^m \1\left[\theta\|\lambda_{t+1}\|_1 + \bfe_i A\lambda_{t+1}   >0 \right]
        \\
        &\leq
        mC_{t+1} \exp(\theta\|\lambda_{t+1}\|_1) \cL(A\lambda_{t+1}).
    \end{align*}
    Next, note
    \begin{align*}
        \|\lambda_{t+1}\|_1
        &= \left\|\lambda_{t_0} + \sum_{i=t_0+1}^{t+1} \alpha_iv_i\right\|_1
        \\
        &\leq \|\lambda_{t_0}\|_1 + \sum_{i=t_0+1}^{t+1}(\tau + \alada{i}).
    \end{align*}
    Combining these facts with the convergence bound from
    \Cref{fact:sep:alada:opt},
    \begin{align*}
        &\sum_{i=1}^m
        \1\left[\frac{-\bfe_i A\lambda_{t+1}}{\|\lambda_{t+1}\|_1} < \theta \right]
        \\
        &\leq mC_{t+1} \exp(\theta\|\lambda_{t_0}\|_1)\cL(A\lambda_{t_0})
        \prod_{i=t_0+1}^{t+1}\Bigg(e^{\theta\tau}
        \left(
            \frac {1+\gamma_i}{1-\gamma_i}
        \right)^{\theta\nu/2}
        \\
        &\quad\cdot\frac {e^\tau C_{i}^4}{2} (1-\gamma_i^2)^{\nu/2}((1+\gamma_i)^{1-\nu} + (1-\gamma_i)^{1-\nu})
        \Bigg).
    \end{align*}
    As in the proof of
    \Cref{fact:sep:alada:opt}, $(\cdot)^{1-\nu}$ is concave, so the $1/2$ may be pushed inside
    this last term to give the vaguely simpler bound
    \begin{align*}
        &\sum_{i=1}^m \1\left[\frac{-\bfe_i A\lambda_{t+1}}{\|\lambda_{t+1}\|_1} < \theta \right]
        \\
        &\leq mC_{t+1} \exp(\theta\|\lambda_{t_0}\|_1)\cL(A\lambda_{t_0})
        \prod_{i=t_0+1}^{t+1}\Bigg(e^{\theta\tau}
        \left(
            \frac {1+\gamma_i}{1-\gamma_i}
        \right)^{\theta\nu/2}
        \\
        &\quad\cdot e^\tau C_{i}^4 (1-\gamma_i^2)^{\nu/2}
        \Bigg).
        \qedhere
    \end{align*}
%%  To reduce this to $\Upsilon_c(\gamma)$, first can choose time step late enough so that
%%  $\theta + \tau < \gamma$  (since $\theta < \gamma$)
%%  and thus
%%  \[
%%      C C^{2\nu\theta}
%%      (1+\epsilon)^{\nu/2}C^3
%%      \leq \left(\frac {1+\gamma}{1-\gamma}\right)^{\tau\nu\theta/2}
%%      \leq \left(\frac {1+\gamma_t}{1-\gamma_t}\right)^{\tau\nu\theta/2},
%%  \]
%%  so this reduces to earlier $\Upsilon_c$ with a slightly beefed up (but still
%%  permissible) $\theta$.
%
%%  the other thing which needs to be shown is that can replace $\gamma_t$ with
%%  $\gamma$..  actually I guess it doesn't matter since we need $\theta < \gamma$
%%  and $\gamma \leq \gamma_t$.  okay.
%
%%  NOTE.  can proof $\Upsilon_c(\gamma) \geq \gamma/2$ by killing the crap term via
%%  concavity and then using the old analysis? maybe?  something?
\end{proof}

\begin{proof}[Proof of \Cref{fact:sep:margins:ada:exploss}]
    Set $\theta := \gamma - \epsilon$, whereby $\theta \in [0,\gamma)$.
    Invoking \Cref{fact:sep:margins:ada:step1} and
    simplifying terms via $t_0 =0$, $C_i = 1$, and $\tau = 0$, then for any $t$,
    \begin{align*}
        &\sum_{i=1}^m \1\left[\frac{-\bfe_i A\lambda_{t+1}}{\|\lambda_{t+1}\|_1} < \theta \right]
        \\
        &\leq
        m\prod_{i=1}^{t+1}\Bigg(
        \left(
            \frac {1+\gamma_i}{1-\gamma_i}
        \right)^{\theta\nu/2}
        \\
        &\quad\cdot\frac{1}{2} (1-\gamma_i^2)^{\nu/2}((1+\gamma_i)^{1-\nu} + (1-\gamma_i)^{1-\nu})
        \\
        &\leq
        m\prod_{i=1}^{t+1}\Bigg(
        \left(
            \frac {1+\gamma}{1-\gamma}
        \right)^{\theta\nu/2}
        \\
        &\quad\cdot\frac{1}{2} (1-\gamma^2)^{\nu/2}((1+\gamma)^{1-\nu} + (1-\gamma)^{1-\nu})
        \Bigg),
    \end{align*}
    where the replacement of $\gamma_i$ by $\gamma$ made use of the first part of
    \Cref{fact:upsilon_goal_in_life}.  Now,
    by the second part of \Cref{fact:upsilon_goal_in_life}, this inner term is less than 1
    iff $\theta< \Upsilon_\nu(\gamma)$.  By \Cref{fact:upsilon}, since $\theta < \gamma$,
    there exists a $\nu$
    sufficiently small that $\Upsilon_\nu(\gamma) > \theta$.
    Consequently, there exists a $T$ so that this product
    is less than $1/m$ whenever $t\geq T$, and the result follows.
\end{proof}

\begin{proof}[Proof of \Cref{fact:sep:margins:ada:alopt}]
    Set $\theta := \gamma - \epsilon$, whereby $\theta \in [0,\gamma)$.
    Since $\ell \in L_\infty$, choose $t_0$ large enough so that
    \[
        C_{t_0}^8 < \left(\frac {1+\gamma}{1-\gamma}\right) ^{\frac{(\gamma-\theta)\nu}4}.
    \]
    By \Cref{fact:alopt_almost_alada}, it follows that the optimal step size satisfies
    \[
        \alada{t} - \tau \leq \alopt{t} \leq \alada{t} + \tau
    \]
    with $\tau = \frac \nu 2 \ln(C_t^4)$.  Combining this with the bound on $C_{t_0}$ above,
    \begin{align*}
        e^{2\tau}C_{t_0}^4  = C_{t_0}^8 <
        \left(\frac {1+\gamma}{1-\gamma}\right) ^{\frac{(\gamma-\theta)\nu}4}.
    \end{align*}
    Plugging this into the general margin bound in \Cref{fact:sep:margins:ada:step1}
    and additionally
    replacing $\gamma_i$
    with $\gamma$ thanks to the first part of \Cref{fact:upsilon_goal_in_life},
    and finally setting $\theta' := \theta + (\gamma-\theta)/2 =  (\theta + \gamma)/2 < \gamma$,
    \begin{align*}
        &\sum_{i=1}^m \1\left[\frac{-\bfe_i A\lambda_{t+1}}{\|\lambda_{t+1}\|_1} < \theta \right]
        \\
        &\leq mC_{t+1} \exp(\theta\|\lambda_{t_0}\|_1)\cL(A\lambda_{t_0})
        \\
        &\quad\cdot\prod_{i=t_0+1}^{t+1}\Bigg(e^{\theta\tau}
        \left(
            \frac {1+\gamma_i}{1-\gamma_i}
        \right)^{\theta\nu/2}
        \\
        &\qquad\cdot\frac {e^\tau C_i^4}{2} (1-\gamma_i^2)^{\nu/2}((1+\gamma_i)^{1-\nu} + (1-\gamma_i)^{1-\nu})
        \Bigg)
        \\
        &\leq mC_{t+1} \exp(\theta\|\lambda_{t_0}\|_1)\cL(A\lambda_{t_0})
        \\
        &\quad\cdot
        \prod_{i=t_0+1}^{t+1}\Bigg(e^{2\tau}C_{t_0+1}^4
        \left(
            \frac {1+\gamma}{1-\gamma}
        \right)^{\theta\nu/2}
        \\
        &\qquad\cdot\frac {1}{2} (1-\gamma^2)^{\nu/2}((1+\gamma)^{1-\nu} + (1-\gamma)^{1-\nu})
        \Bigg)
        \\
        &\leq m C_{t+1} \exp(\theta\|\lambda_{t_0}\|_1)\cL(A\lambda_{t_0})
        \\
        &\quad\cdot
        \prod_{i=t_0+1}^{t+1}\Bigg(
        \left(
            \frac {1+\gamma}{1-\gamma}
        \right)^{\theta'\nu/2}
        \\
        &\qquad\cdot\frac {1}{2} (1-\gamma^2)^{\nu/2}((1+\gamma)^{1-\nu} + (1-\gamma)^{1-\nu})
        \Bigg).
    \end{align*}
    By the second part of \Cref{fact:upsilon_goal_in_life}, the term within the product is
    less than one, and thus for all large $t$, this entire bound is less than $1$, which
    gives the result.
\end{proof}

\begin{proof}[Proof of \Cref{fact:sep:margins:ada:vague_rate}]
    To start, note that, for any $t\geq 0$,
    \begin{align}
        (1-\gamma_t)^{1-\theta}(1+\gamma_t)^{1+\theta}
        &=
        (1-\gamma_t^2)^{1-\theta}(1+\gamma_t)^{2\theta}
        \notag\\
        &\leq
        \exp\left(
            -\gamma_t^2(1-\theta) + \gamma_t(2\theta)
        \right)
        \notag\\
        &=
        \exp\left(
            -\gamma_t^2 + \theta\gamma_t(2+\gamma_t)
        \right).
        \label{eq:fr_sc:margins:2}
    \end{align}
    Next, since
    \[
        \theta
        \leq \frac {\gamma}{1+\gamma}
        = \frac {1}{1+1/\gamma}
        \leq \frac {1}{1+1/\gamma_t}
        = \frac{\gamma_t}{1+\gamma_t},
    \]
    then $\theta \leq \gamma / (1+\gamma)$ implies
    \begin{align*}
        \frac {d}{d\gamma_t}\left(
            -\gamma_t^2 + \theta\gamma_t(2+\gamma_t)
        \right)
        &= -2\gamma_t + 2\theta(1 + \gamma_t)
        \\
        &\leq -2\gamma_t + 2\gamma_t
        \\
        &= 0.
    \end{align*}
    In particular, the expression $-\gamma_t^2 + \theta\gamma_t(2+\gamma_t)$ is decreasing
    in $\gamma$, and thus $\gamma_t\geq \gamma$ implies
    \[
            -\gamma_t^2 + \theta\gamma_t(2+\gamma_t)
            \leq -\gamma^2 + \theta\gamma(2+\gamma),
    \]
    and consequently, combined with the bound in \eqref{eq:fr_sc:margins:2},
    \[
        (1-\gamma_t)^{1-\theta}(1+\gamma_t)^{1+\theta}
        \leq \exp(-\gamma^2 + \theta\gamma(2+\gamma)).
    \]
    Plugging this into the simplified generic bound in \Cref{fact:sep:margins:ada:step1}
    with the specialization $\ell = \exp$, $\tau = 0$, $C_i = 1$, and $t_0 = 0$,
    it follows that
    \begin{align*}
        &\sum_{i=1}^m \1\left[\frac{-\bfe_i A\lambda_{t+1}}{\|\lambda_{t+1}\|_1} < \theta \right]
        \\
        &\leq m
        \prod_{i=1}^{t+1}\Bigg(
        \left(
            \frac {1+\gamma_i}{1-\gamma_i}
        \right)^{\theta\nu/2}
        (1-\gamma_i^2)^{\nu/2}
        \Bigg)
        \\
        &\leq
        m\exp\left(
            -\frac {\nu(t+1)}{2}
            (-\gamma^2 + \theta\gamma(2+\gamma))
        \right).
    \end{align*}
    The rest of the result follows by noting
    $\theta < \gamma / (2+\gamma)$ implies
    $-\gamma^2 + \theta\gamma(2+\gamma)<0$,
    whereby choices
    \[
        t > \frac {2\ln(m)}{\nu(\gamma^2 - \theta\gamma(2+\gamma))}
    \]
    exist, and plugging this all in to the above bound
    grants that $\cM(A\lambda_t) \geq \theta$.
\end{proof}

\begin{proof}[Proof of \Cref{fact:sep:margins:basic}]
    For $\alada{t}$ and $\alopt{t}$,
    \Cref{fact:sep:margins:ada:exploss}
    and
    \Cref{fact:sep:margins:ada:alopt}
    already state the results in the desired asymptotic form.

    For the other two, since $\ell\in\bL_\infty$, $t_0$ can be chosen sufficiently large so
    that $C_{t_0}$ is arbitrarily close to $1$,
    whereby the bounds in
    \Cref{fact:sep:margin:quadub}
    and
    \Cref{fact:sep:margin:wolfe}
    become sufficiently tight by taking $\nu$ small and $t\nu$ large.
\end{proof}

\subsubsection{The Quantity $\Upsilon_\nu$}

\begin{definition}
    Define
    \begin{align*}
        &\Upsilon_\nu(\gamma) :=
        \\
        &
        \frac
        {\frac 2 \nu \ln(2) - \frac 2 \nu \ln((1+\gamma)^{1-\nu} + (1-\gamma)^{1-\nu})
        - \ln(1-\gamma^2)}
        {\ln(1+\gamma) - \ln(1-\gamma)};
        %\qedhere
    \end{align*}
    in the case that $\nu=1$, this quantity has been extensively studied
    in the context of AdaBoost's margins
    \citep{warmuth_maxmargin_boosting,rudin_adaboost_bad_margin,schapire_freund_book_final}
\end{definition}

The basic properties of $\Upsilon_\nu$ are as follows.

\begin{theorem}
    \label{fact:upsilon}
    Suppose $\gamma \in (0,1)$.

    \begin{enumerate}
        \item $\gamma/2 \leq \Upsilon_\nu(\gamma) \leq \gamma$.
            %\red{(this is partially known and i need to cite appropriately.)}
            \label{fact:upsilon:1}
       %\item $\Upsilon_c(\gamma) \leq 
       %    \min\{\ \gamma\ ,\ \Upsilon_{1/2}(\gamma) + 2\ln(2) / (1+\ln(2))\ ,
       %    \ \Upsilon_{1/2}(\gamma)/(2c)\ \}$.
       %    \label{fact:upsilon:2}
        \item $\lim_{\nu\downarrow 0} \Upsilon_\nu(\gamma) = \gamma$.
            \label{fact:upsilon:3}
    \end{enumerate}
\end{theorem}

The bounds $\gamma/2 \leq \Upsilon_1(\gamma) \leq \gamma$ were known in the case that
$\nu = 1$ (cf. \citet{warmuth_maxmargin_boosting}
and \citet[Bibliographic Notes, Chapter 5]{schapire_freund_book_final}).

\begin{proof}
    (\Cref{fact:upsilon:1}, subcase $\Upsilon_\nu(\gamma)\geq\gamma/2$.)
    To start, note that $(\cdot)^{1-\nu}$ is a concave function, whereby
    \begin{align*}
        &
        - \frac 2 \nu \ln\left(
            (1+\gamma)^{1-\nu} + (1-\gamma)^{1-\nu}
        \right)
        \\
        &=
        - \frac 2 \nu \ln\left(2\left(
            \frac 1 2 (1+\gamma)^{1-\nu} + \frac 1 2 (1-\gamma)^{1-\nu}
        \right)
        \right)
        \\
        &\geq
        - \frac 2 \nu \ln\left(2\left(1\right)^{1-\nu}\right)
        \\
        &=
        - \frac 2 \nu \ln\left(2\right).
    \end{align*}
    It follows that
    \begin{align*}
        \Upsilon_\nu(\gamma)
        \geq
        \frac
        { - \ln(1-\gamma^2)}
        {\ln(1+\gamma) - \ln(1-\gamma)}
        = \Upsilon_{1}(\gamma).
    \end{align*}

    Next recall the series expansion
    \[
        \ln(1+z) = \sum_{n=1}^\infty \frac {(-1)^{n+1}}{n} z^n
    \]
    (when $|z| < 1$).  Plugging this in to the simplified form
    of $\Upsilon_1(\gamma)$ and paying attention to cancellations in the numerator and
    denominator (odd and even terms, respectively),
    \begin{align*}
        \Upsilon_{1}(\gamma)
        &=
        \frac
        {- \sum_{n=1}^\infty \frac {(-1)^{n+1}}{n} (\gamma)^n
        - \sum_{n=1}^\infty \frac {(-1)^{n+1}}{n} (-\gamma)^n}
        {\sum_{n=1}^\infty \frac {(-1)^{n+1}}{n} (\gamma)^n
        - \sum_{n=1}^\infty \frac {(-1)^{n+1}}{n} (-\gamma)^n}
        \\
        &=
        \frac
        {- 2\sum_{n=1}^\infty \frac {(-1)^{2n+1}}{2n} (\gamma)^{2n}}
        {2\sum_{n=1}^\infty \frac {(-1)^{2n-1+1}}{2n-1} (\gamma)^{2n-1}}
        \\
        &=
        \frac
        {\gamma\sum_{n=1}^\infty \frac {1}{2n} (\gamma)^{2n}}
        {\sum_{n=1}^\infty \frac {1}{2n-1} (\gamma)^{2n}}.
    \end{align*}
    To finish, note that $n\geq 1$ implies $1/(4n-2) \leq 1/(2n) \leq 1/(2n-1)$, and thus
    \begin{align*}
        \frac \gamma 2
        &=
        \frac
        {\gamma\sum_{n=1}^\infty \frac {1}{2(2n-1)} (\gamma)^{2n}}
        {\sum_{n=1}^\infty \frac {1}{2n-1} (\gamma)^{2n}}
        \\
        &\leq
        \frac
        {\gamma\sum_{n=1}^\infty \frac {1}{2n} (\gamma)^{2n}}
        {\sum_{n=1}^\infty \frac {1}{2n-1} (\gamma)^{2n}}
        \\
        &\leq
        \frac
        {\gamma\sum_{n=1}^\infty \frac {1}{2n-1} (\gamma)^{2n}}
        {\sum_{n=1}^\infty \frac {1}{2n-1} (\gamma)^{2n}}
        \\
        &= \gamma.
    \end{align*}
    That is to say, $\gamma/2 \leq \Upsilon_{1}(\gamma) \leq \gamma$,
    which combined with the above also gives $\Upsilon_\nu(\gamma)
    \geq \Upsilon_1(\gamma) \geq \gamma/2$.

    (\Cref{fact:upsilon:1}, subcase $\Upsilon_\nu(\gamma)\leq\gamma$.)
    By the power mean inequality \citep[Equation 8.12]{c-s_master},
    \begin{align*}
        &\left(\frac{1+\gamma}{2}(1+\gamma)^{-\nu}
        + \frac{1-\gamma}{2}(1-\gamma)^{-\nu}\right)^{-1/\nu}
        \\
        &\leq (1+\gamma)^{ \frac{1+\gamma}{2}}
        (1-\gamma)^{ \frac{1-\gamma}{2}}.
    \end{align*}
    It follows that
    \begin{align*}
        &- \frac {2}{\nu} \ln\left(
        \frac{1+\gamma}{2}(1+\gamma)^{-\nu}
        + \frac{1-\gamma}{2}(1-\gamma)^{-\nu}
        \right)
        \\
        &\leq
        (1+\gamma)\ln(1+\gamma) + (1-\gamma)\ln(1-\gamma).
    \end{align*}
    As such,
    \begin{align*}
        &\Upsilon_\nu(\gamma)
        \\
        &\leq
        \frac
        {
            (1+\gamma)\ln(1+\gamma) + (1-\gamma)\ln(1-\gamma)
        }
        {\ln(1+\gamma)-\ln(1-\gamma)}
        \\
        & \quad+
        \frac
        {
            -\ln(1+\gamma) - \ln(1-\gamma)
        }
        {\ln(1+\gamma)-\ln(1-\gamma)}
        \\
        &= \frac
        {
            \gamma\ln(1+\gamma) -\gamma\ln(1-\gamma)
        }
        {\ln(1+\gamma)-\ln(1-\gamma)}
        \\
        &= \gamma.
    \end{align*}

    (\Cref{fact:upsilon:3}.)
    Consider the (halved, negated) first term
    \begin{align*}
        &\frac {\ln((1+\gamma)^{1-\nu} + (1-\gamma)^{1-\nu}) - \ln(2)}{\nu}
        \\
        &= \frac {\ln(0.5(1+\gamma)^{1-\nu} + 0.5 (1-\gamma)^{1-\nu})}{\nu}.
    \end{align*}
    By l'H\^opital's rule,
    \begin{align*}
        &\lim_{\nu\to 0} \frac {\ln(0.5(1+\gamma)^{1-\nu} + 0.5 (1-\gamma)^{1-\nu})}{\nu}
        \\
        &=
        \lim_{\nu\to 0}
        \frac{-(1+\gamma)^{1-\nu}\ln(1+\gamma) - (1-\gamma)^{1-\nu}\ln(1-\gamma)}
        {(1+\gamma)^{1-\nu} + (1-\gamma)^{1-\nu}}
        \\
        &=
        -\frac 1 2 \left((1+\gamma)\ln(1+\gamma) + (1-\gamma)\ln(1-\gamma)\right)
        .
    \end{align*}
    Consequently (recalling that this term was both halved and negated)
    \begin{align*}
        \lim_{\nu\to 0} \Upsilon_\nu(\gamma)
        &= \frac
        {
            (1+\gamma)\ln(1+\gamma) + (1-\gamma)\ln(1-\gamma)
        }
        {\ln(1+\gamma)-\ln(1-\gamma)}
        \\
        &\quad +  \frac
        {
            -\ln(1+\gamma) - \ln(1-\gamma)
        }
        {\ln(1+\gamma)-\ln(1-\gamma)}
        \\
        &
        = \gamma.
        \qedhere
    \end{align*}
\end{proof}

The usefulness of $\Upsilon_\nu$ is captured in the following
\namecref{fact:upsilon_goal_in_life}.

\begin{lemma}
    \label{fact:upsilon_goal_in_life}
    Let $\nu\in (0,1]$ and $\theta \in [0,1]$ be given.
    The map
    \begin{align*}
        \gamma
        &\mapsto \left(
            \frac {1+\gamma}{1-\gamma}
        \right)^{\theta\nu/2}(1-\gamma^2)^{\nu/2}
        \\
        &\quad\cdot((1+\gamma)^{1-\nu} + (1-\gamma)^{1-\nu})
    \end{align*}
    is nonincreasing over $[\theta,1]$.  Additionally,
    now taking $\gamma$ to be fixed,
    $\theta < \Upsilon_\nu(\gamma)$
    iff
    \begin{align*}
        &\frac 1 2\left(
            \frac {1+\gamma}{1-\gamma}
        \right)^{\theta\nu/2}(1-\gamma^2)^{\nu/2}
        \\
        &\quad\cdot ((1+\gamma)^{1-\nu} + (1-\gamma)^{1-\nu})
        \\
        &< 1.
    \end{align*}
\end{lemma}
\begin{proof}
    Let $f(\gamma)$ be the prescribed map.  To establish $f$ is nonincreasing,
    it will be shown that each element of the product $f(\gamma) = g(\gamma)h(\gamma)$
    is nonincreasing, where
    \begin{align*}
        g(\gamma) &:= \left(
            \frac {1+\gamma}{1-\gamma}
        \right)^{\theta\nu/2}(1-\gamma^2)^{\nu/2}
        \\
        h(\gamma) &:=(1+\gamma)^{1-\nu} + (1-\gamma)^{1-\nu}.
    \end{align*}
    First, set $\nu' := \nu/2$, and note
    \begin{align*}
        g'(\gamma)
        &= \frac {d}{d\gamma} (1+\gamma)^{\nu'(1+\theta)}(1-\gamma)^{\nu'(1-\theta)}
        \\
        &=\nu'(1+\theta)(1+\gamma)^{\nu'(1+\theta)-1}(1-\gamma)^{\nu'(1-\theta)}
        \\
        &\quad-
        \nu'(1-\theta)(1-\gamma)^{\nu'(1-\theta)-1}(1+\gamma)^{\nu'(1+\theta)}
        \\
        &= \nu'(1+\gamma)^{\nu'(1+\theta)-1}(1-\gamma)^{\nu'(1-\theta)-1}
        \\
        &\quad\cdot\left(
            (1+\theta)(1-\gamma) - (1-\theta)(1+\gamma)
        \right)
        \\
        &= 2\nu'(1+\gamma)^{\nu'(1+\theta)-1}(1-\gamma)^{\nu'(1-\theta)-1}
        \left(
           \theta - \gamma
        \right),
    \end{align*}
    where this last term is nonpositive since $\theta \leq \gamma$.  Consequently,
    $g(\gamma)$ is nonincreasing.

    For $h(\gamma)$, note similarly that
    \begin{align*}
        h'(\gamma)
        &=
        (1-\nu) \left( (1+\gamma)^{-\nu} - (1-\gamma)^{-\nu}\right)
        \\
        &=
        \frac{1-\nu}{(1+\gamma)^{\nu}(1-\gamma)^{\nu}}\left(
            (1-\gamma)^{\nu} - (1+\gamma)^{\nu}
        \right)
        \\
        &\leq 0.
    \end{align*}
    Together $f(\gamma) = g(\gamma)h(\gamma)$ is nonincreasing in $\gamma$.

    For the second statement, note that
    \begin{align*}
        1
        &> \frac 1 2
        \left(\frac{1+\gamma}{1-\gamma}\right)^{\theta\nu/2}
        (1-\gamma^2)^{\nu/2}
        \\
        &\quad\cdot((1+\gamma)^{1-\nu} + (1-\gamma)^{1-\nu})
    \end{align*}
    is equivalent %(by strict monotonicity of $\ln(\cdot)$) to
    to
    \begin{align*}
       %&\iff
       %\\
        0
        &>
        -\ln(2)
        +\frac \nu 2\theta \ln\left(\frac{1+\gamma}{1-\gamma}\right)
        +\frac \nu 2\ln(1-\gamma^2)
        \\
        &\quad+\ln((1+\gamma)^{1-\nu} + (1-\gamma)^{1-\nu})
       %&\iff
       %\\
    \end{align*}
    is equivalent to
    \begin{align*}
        &\theta
        <\\&
        \frac
        {
        \ln(2)
        -\frac \nu 2\ln(1-\gamma^2)
        -\ln((1+\gamma)^{1-\nu} + (1-\gamma)^{1-\nu})
        }
        {
        \frac \nu 2 \ln\left(\frac{1+\gamma}{1-\gamma}\right)
        },
        %= \Upsilon_\nu(\gamma),
    \end{align*}
    where the last expression can be written $\theta < \Upsilon_\nu(\gamma)$.
\end{proof}

\subsubsection{Miscellaneous Technical Material}
\begin{lemma}
    \label{fact:alopt_almost_alada}
    Suppose $A\in\{-1,+1\}^{m \times n}$ is binary and $\ell \in \bL$.
    Then
    \begin{align*}
        \frac 1 2 \ln\left(
            \frac {1+\gamma_t}
            {1-\gamma_t}
        \right)
        - \frac 1 2 \ln(C_t^4)
        &\leq \aloptT{t}{1}
        \\
        &\leq
        \frac 1 2 \ln\left(
            \frac {1+\gamma_t}
            {1-\gamma_t}
        \right)
        + \frac 1 2 \ln(C_t^4).
    \end{align*}
    More simply,
    \[
        \left|
        \alopt{t} - \alada{t}
        \right|
        \leq \frac {\nu}{2} \ln\left(C_t^4\right).
    \]
  %%In particular, if
  %%$((1-\gamma)/(1+\gamma))^{\tau/4} \leq C_t \leq ((1+\gamma)/(1-\gamma))^{\tau/4}$
  %%for some $\tau > 0$,
  %%then \red{scale by $\nu$ and rewrite this.. also indexing lelz}
  %%\[
  %%    \frac {1-\tau} 2 \ln\left(
  %%        \frac {1+\gamma_t}
  %%        {1-\gamma_t}
  %%    \right)
  %%    \leq \alopt{t+1}
  %%    \leq
  %%    \frac {1+\tau} 2 \ln\left(
  %%        \frac {1+\gamma_t}
  %%        {1-\gamma_t}
  %%    \right)
  %%    .
  %%\]
\end{lemma}
\begin{proof}
  %%\red{\textbf{separate out the stuff about the sum of pos weights over sum of
  %%        negative weights?  because that is how schapire wrote his stepsize in the
  %%logistic loss boosting stuff.}}
    Choose $s\in \{\pm 1\}$ so that $v_{t+1} = \bfe_j s$ for some $\bfe_j$.
    Then, by first order conditions on the
    optimal step size, and adopting shorthand notation where the summations take $j$ fixed
    according to the preceding text, but $i\in[m]$ may vary,
    \begin{align*}
        0
        &=
        \sum_{A_{ij} < 0} s A_{ij}\ell'(\bfe_i^\top A(\lambda_t + s\alpha_{t+1} \bfe_j))
        \\
        &\quad +
        \sum_{A_{ij} > 0} s A_{ij}\ell'(\bfe_i^\top A(\lambda_t + s\alpha_{t+1} \bfe_j))
        \\
        &\leq
        C_{t+1}^{-1} \sum_{A_{ij} < 0} s A_{ij}\exp(\bfe_i^\top A\lambda_t) \exp(s\alpha_{t+1} A_{ij})
        \\
        &\quad +
        C_{t+1}^{1} \sum_{A_{ij} > 0} s A_{ij}\exp(\bfe_i^\top A\lambda_t)\exp(s\alpha_{t+1} A_{ij})
        \\
        &\leq
        \exp(-s\alpha_{t+1}) C_{t+1}^{-2} \sum_{A_{ij} < 0} s A_{ij}\ell'(\bfe_i^\top A\lambda_t)
        \\&\quad
        +
        \exp(s\alpha_{t+1}) C_{t+1}^{2} \sum_{A_{ij} > 0} s A_{ij}\ell'(\bfe_i^\top A\lambda_t),
    \end{align*}
    which can be rearranged to yield
    \begin{align*}
        s\alpha_{t+1}
        &\geq
        \frac 1 2 \ln\left(
            \frac {-sC_{t+1}^{-2}\sum_{A_{ij}<0}A_{ij}\ell'(\bfe_i^\top A\lambda_t)}
            {sC_{t+1}^{2}\sum_{A_{ij}>0}A_{ij}\ell'(\bfe_i^\top A\lambda_t)}
        \right)
        \\
        &=
        \frac 1 2 \ln\left(
            \frac {\sum_{A_{ij}<0}\ell'(\bfe_i^\top A\lambda_t)}
            {\sum_{A_{ij}>0}\ell'(\bfe_i^\top A\lambda_t)}
        \right)
        - \frac 1 2 \ln(C_{t+1}^4).
    \end{align*}
    To simplify further, note that
    \begin{align*}
        s\gamma_{t+1}
        &= \frac {-\sum_{i=1}^m A_{ij} \ell'(\bfe_i^\top A\lambda_t)}
        { \|\nabla \cL(A\lambda_t)\|_1}
        \\
        &= \frac {\sum_{A_{ij}<0} \ell'(\bfe_i^\top A\lambda_t)
-\sum_{A_{ij}>0} \ell'(\bfe_i^\top A\lambda_t)}
        { \|\nabla \cL(A\lambda_t)\|_1},
        \\
        1
        &= \frac {\sum_{A_{ij}<0}  \ell'(\bfe_i^\top A\lambda_t)
        +\sum_{A_{ij}>0} \ell'(\bfe_i^\top A\lambda_t)}
        { \|\nabla \cL(A\lambda_t)\|_1},
    \end{align*}
    which can be added and subtracted to yield
    \[
        \frac {\sum_{A_{ij}>0}\ell'(\bfe_i^\top A\lambda_t)}
        {\sum_{A_{ij}<0}\ell'(\bfe_i^\top A\lambda_t)}
        = \frac {1-s\gamma_{t+1}}{1+s\gamma_{t+1}},
    \]
    whereby
    \[
        s\alpha_{t+1}
        \geq
        \frac 1 2 \ln\left(
            \frac {1+s\gamma_{t+1}}{1-s\gamma_{t+1}}
        \right)
        - \frac 1 2 \ln(C_{t+1}^4).
    \]
    Repeating the steps above to prove a lower bound on $\alpha_{t+1}$,
    it also follows that
    \[
        s\alpha_{t+1}
        \leq
        \frac 1 2 \ln\left(
            \frac {1+s\gamma_{t+1}}{1-s\gamma_{t+1}}
        \right)
        + \frac 1 2 \ln(C_{t+1}^4).
    \]
    To finish the first part of the result, it suffices to consider the
    cases $s=+1$ and $s=-1$ separately, which both lead to the desired pair of
    inequalities.

    For the second guarantee, first note that $\alopt{t} = \nu\aloptT{t}{1}$
    and $\alada{t} = \nu\aladaT{t}{1}$, and so recalling the form of $\aladaT{t}{1}$ and
    scaling the first guarantee by $\nu$, it follows that
    \[
        |\alopt{t} - \alada{t}| \leq \frac \nu 2\ln(C_{t}^4).
        \qedhere
    \]
    %%didn't need this weird crap
 %%%For the second part, first note that $\alopt{t+1} = \nu\aloptT{t+1}1$,
 %%%and thus
 %%%\begin{align*}
 %%%    - \frac \nu 2 \ln(C_t^4)
 %%%    &\leq \aloptT{t+1}{\nu}
 %%%    -\frac \nu 2 \ln\left(
 %%%        \frac {1+\gamma_t}
 %%%        {1-\gamma_t}
 %%%    \right)
 %%%    \leq
 %%%     \frac \nu 2 \ln(C_t^4).
 %%%\end{align*}
 %%%Next, note that
 %%%$\gamma \leq \gamma_t$ implies
 %%%\[
 %%%    \frac {1+\gamma_t}{1-\gamma_t}
 %%%    \geq
 %%%    \frac {1+\gamma}{1-\gamma}.
 %%%\]
 %%%Consequently, the assumption $..$ gives
 %%%\begin{align*}
 %%%     \frac \nu 2 \ln(C_t^4)
 %%%     &\leq
 %%%\end{align*}
 %%%whereby the result follows. \red{add stuff for $\nu$..}
\end{proof}

\section{Deferred Material from \Cref{sec:general}}

\begin{proof}[Proof sketch of \Cref{fact:general:opt}]
    All the convergence rates developed by \Citet[Section 6]{primal_dual_boosting_arxiv}
    stem from an inequality
    \begin{align*}
        &\cL(A\lambda_{t+1}) - \bar\cL_A
        \\
        &\leq
        (\cL(A\lambda_t) -\bar\cL_A)\left(
            1-\frac {\|A^\top\nabla \cL(A\lambda_t)\|_\infty^2}{c\cL(A\lambda_t)(\cL(A\lambda_t)
            - \bar\cL_A)}
        \right),
    \end{align*}
    where $c>0$ is some constant independent of $t$ (or improving with $t$, in which case
    the bound may be worsened by taking the choice for $t=0$)
    \citep[Proposition 6.2, Proposition D.6]{primal_dual_boosting_arxiv}.
    Exactly such a bound was provided
    for each line search in the proof of its respective optimization guarantee in the
    separable case
    (cf.
    \Cref{fact:sep:opt:quadub},
    \Cref{fact:sep:opt:wolfe};
    no need to adjust
    \Cref{fact:sep:opt:ada}, since $\ell = \exp$ and $A$ binary causes
    $\alada{t}=\alopt{t}$, and so \Cref{fact:sep:opt:quadub} covers this case).
    Replacing $c$ with the particulars for each step size will only impact the final
    rates in Theorems 6.3, 6.6, and 6.12 by these constants.  The only other thing to
    check is that $\ell \in \bG$, the class of losses considered 
    by \citet[Section 6]{primal_dual_boosting_arxiv}; it can be checked directly
    that $\bL \subset \bG$.
\end{proof}

In order to establish the margin properties, the following
\namecref{fact:general:margins:hatgamma} is essential.
\begin{lemma}
    \label{fact:general:margins:hatgamma}
    Consider the setting of \Cref{fact:general:margins}.
    Then there exists $T$ and $\hat \gamma$ so that,
    for all $t\geq T$,
    \[
        \frac {\|A^\top \nabla \cL(A\lambda_t)\|_\infty}{\cL(A\lambda_t) - \bar\cL_A}
        \geq \hat \gamma.
    \]
\end{lemma}
\begin{proof}[Proof sketch]
    As discussed in the proof of \Cref{fact:general:opt},
    the results of \citet{primal_dual_boosting_arxiv}, which are superficially specialized
    to the Wolfe line search, carry over for the other line searches here with only a change
    of constants; consequently, those results carry over wholesale.

    To start, let $S$ be a compact cube containing all iterates,
    and let $\gamma(A,S)$ be the corresponding generalized weak learning rate
    \citet[Definition 4.3]{primal_dual_boosting_arxiv}.

    By \citep[Theorem 5.9]{primal_dual_boosting_arxiv}, $\cL + \iota_{\im(A_+)}$ (i.e.,
    the function which is $\cL(y)$ when $y=A_+\lambda$ for some $\lambda \in \R^n$, and
    $\infty$ otherwise) has
    compact level sets, and thus strict convexity of $\cL$ grants a modulus of strong
    convexity $c>0$ over $S$;
    furthermore, it holds for every $t$ that
    \begin{align*}
        &\cL(A_+\lambda_t) - \bar \cL_A
        \\
        &\quad\leq \frac 1{2c} \left\|\nabla \cL(A_+\lambda_t)
        - \sfP^1_{\nabla \cL(S)\cap\ker(A^\top_+)}(\nabla \cL(A_+\lambda_t))\right\|_1^2,
    \end{align*}
    where $\sfP^1_{\nabla\cL(S)\cap \ker(A^\top_+)}$ denotes the $l^1$ projection onto
    $\nabla \cL(S)\cap \ker(A^\top_+)$, the latter being the kernel (nullspace) of $A^\top_+$
    \citep[Lemma 6.8]{primal_dual_boosting_arxiv}.

    Now choose $T$ so that, for every $t\geq T$,
    \[
        \cL(A_+\lambda_t) - \bar \cL_A \leq  \cL(A\lambda_t) - \bar \cL \leq 2c,
    \]
    which is possible by the convergence of $\{\lambda_t\}_{t=1}^\infty$
    (cf. \Cref{fact:general:opt} or \citep[Theorem 6.12]{primal_dual_boosting_arxiv}).

    Using these facts, the definition of $\gamma(A,S)$,
    the choice $\phi = \exp$, and the fact
    $\inf_\lambda \cL(A_+\lambda) = \inf_\lambda \cL(A\lambda) = \bar \cL_A$
    \citep[Theorem 5.9]{primal_dual_boosting_arxiv},
    \begin{align*}
        &\frac {\|A^\top \nabla \cL(A\lambda_t)\|_\infty}{\cL(A\lambda_t) - \bar\cL_A}
        \\
        &\geq
        \gamma(A,S)
        \left(
        \frac {\|\nabla\cL(A\lambda_t) - \sfP^1_{S\cap \ker(A^\top)}(\nabla\cL(A\lambda_t))\|_1}
        {\cL(A\lambda_t) - \bar \cL_A}
        \right)
        \\
        &=
        \gamma(A,S)
        \frac{C_T^2}{C_T^2}
        \Bigg(
        \frac {\|\nabla\cL(A_0\lambda_t)\|_1}
        {\cL(A\lambda_t) - \bar \cL_A}
        \\
        &\qquad
        + \frac{\|\nabla\cL(A_+\lambda_t) - \sfP^1_{S\cap \ker(A_+^\top)}(\nabla\cL(A_+\lambda_t))\|_1}
        {\cL(A\lambda_t) - \bar \cL_A}
        \Bigg)
        \\
        &\geq
        \frac{\gamma(A,S)}{C_T^2}
        \left(
        \frac {\cL(A_0\lambda_t) + \sqrt{2c(\cL(A_+\lambda_t) -\bar \cL_A)}}
        {\cL(A\lambda_t) - \bar \cL_A}
        \right)
        \\
        &\geq
        \frac{\gamma(A,S)}{C_T^2}
        \Bigg(
        \frac {\cL(A_0\lambda_t)}
        {\cL(A\lambda_t) - \bar \cL_A}
        \\
        &\qquad + \frac{\sqrt{(\cL(A_+\lambda_t) - \bar \cL_A)(\cL(A_+\lambda_t) -\bar \cL_A)}}
        {\cL(A\lambda_t) - \bar \cL_A}
        \Bigg)
        \\
        &= \frac{\gamma(A,S)}{C_T^2}.
    \end{align*}
    To finish, set $\hat\gamma := \gamma(A,S) / C_T^2$.
\end{proof}

Another technical lemma is helpful.
\begin{lemma}
    \label{fact:general:margins:big_norms}
    Consider the setting of \Cref{fact:general:margins}.
    For each step size choice and $B>0$, there exists $T_B$
    so that for all $t\geq T_B$, $\|\lambda_t\|_1 \geq B$.
\end{lemma}
\begin{proof}[Proof sketch]
    This follows from \Cref{fact:general:opt} and $|H(A)| < m$.
    In particular, choose any example $i \in H(A)^c$; there
    exists $\epsilon > 0$ so that $\frac 1 m \ell(\bfe_i A^\top\lambda)<\epsilon$
    (which is a necessary condition for $\cL(A\lambda) < \epsilon$)
    only when $\bfe_i^\top A\lambda\leq -B\|\bfe_i^\top A\|_\infty$, and
    so the result follows by combining this with H\"older's inequality, namely
    the inequality
    $-\bfe_i^\top A\lambda \leq \|\bfe_i^\top A\|_\infty \|\lambda\|_1$;
    the optimality guarantee provides that this holds for all large $t$.
\end{proof}

In order to proof the margin results, it is helpful to split into two cases,
one being the Wolfe step sizes, the other being a generalization of the quadratic
upper bound step sizes.

\begin{lemma}
    \label{fact:general:margins:quadub}
    Consider the setting of \Cref{fact:general:margins},
    but with step sizes $0.5\alqub{i} \leq \alpha_i \leq 1.5\alqub{i}$.
    Then there exists $\hat\gamma>0$ and $T$ so that,
    for all $t\geq T$, all margins (over $H(A)^c$) exceed $\hat\gamma$.
\end{lemma}
\begin{proof}[Proof sketch]
    Consider the quadratic upper bound line search in
    \Cref{fact:sep:opt:quadub} and its proof.  It is unclear whether $0.5\alqub{i}$ or
    $1.5\alqub{i}$ give a better step, due to the term $\nu$.  However,
    since $\alqubT{i}{1}$ is the minimizer, symmetry grants that
    $0.5\alqub{t}$ is guaranteed to be a worse choice than anything in the
    specified interval.  As such,  plugging this in to the quadratic upper bound
    yields
    \[
        \cL(A\lambda_{t+1})
        \leq
        \cL(A\lambda_t)
        - \frac {c_0\gamma_{t+1} \|A^\top \nabla\cL(A\lambda_t)\|_\infty}{2}.
    \]
    for some constant $c_0>0$ depending on $C_1$ and not on $t$.

    Now choose $T_1$ according to
    \Cref{fact:general:margins:hatgamma};
    by the above and \Cref{fact:general:margins:hatgamma}, for any $t\geq T_1$,
    \begin{align*}
        &\cL(A\lambda_{t+1}) - \bar \cL_A
        \\
        &\leq
        \cL(A\lambda_t) - \bar \cL_A
        - \frac {\gamma_{t+1} c_0\|A^\top \nabla\cL(A\lambda_t)\|_\infty}{2}
        \\
        &\leq
        (\cL(A\lambda_t) - \bar \cL_A)\left(
        1  - \frac {\gamma_{t+1} c_0\|A^\top \nabla\cL(A\lambda_t)\|_\infty}
        {2(\cL(A\lambda_t - \bar \cL_A)}
        \right)
        \\
        &\leq
        (\cL(A\lambda_t) - \bar \cL_A)\left(
        1  - \frac {\gamma_{t+1} c_0 \hat \gamma}{2}
        \right),
    \end{align*}
    which, after recursive application, provides
    \[
        \cL(A\lambda_{t+1}) - \bar \cL_A
        \leq
        (\cL(A\lambda_{T_1}) - \bar \cL_A)\exp\left(
        -\frac {\hat \gamma c_0}{2} \sum_{i=T_1+1}^{t+1} \gamma_i
        \right).
    \]
    Since
    \begin{align*}
        \|\lambda_{t+1}\|_1
        &= \|\lambda_{T_1} + \sum_{i=T_1}^{t+1} \alpha_i v_i\|_1
        \\
        &\leq \|\lambda_{T_1}\|_1 + \sum_{i=T_1+1}^{t+1}\alpha_i
        \\
        &\leq \|\lambda_{T_1}\|_1 + 1.5\nu\sum_{i=T_1+1}^{t+1}\gamma_i,
    \end{align*}
    it follows that
    \begin{align*}
        &\cL(A\lambda_{t+1}) - \bar \cL_A
        \\
        &\leq
        (\cL(A\lambda_{T_1} - \bar \cL_A))\exp\left(
        -\frac {\hat \gamma c_0}{3\nu} (\|\lambda_{t+1}\|_1 - \|\lambda_{T_1}\|_1)
        \right).
    \end{align*}
    For any iteration $t$, let $b_t \in \{\bfe_i\}_{i=1}^{m_0}$ index
    any example in $[m]\setminus H(A)$ which achieves the worst margin (amongst
    elements off the hard core) for this iteration.
    Since the optimal error on this example is 0
    \citep[Theorem 5.9]{primal_dual_boosting_arxiv},
    for any $t > T_1$,
    \begin{align*}
        &\exp(b_t^\top A\lambda_t)
        \\
        &=
        \exp(b_t^\top A\lambda_t) - 0
        \\
        &\leq m C_{T_1}(\cL(A\lambda_t) - \bar \cL_A)
        \\
        &\leq
        m C_{T_1}(\cL(A\lambda_{T_1} - \bar \cL_A))\exp\left(
        -\frac {\hat \gamma c_0}{3\nu} (\|\lambda_{t}\|_1 - \|\lambda_{T_1}\|_1)
        \right)
        \\
        &=
        \exp\left(
            -\hat \gamma c_0\|\lambda_{t}\|_1/(3\nu)
        \right)
        \\
        &\quad\cdot\underbrace{C_t(\cL(A\lambda_{T_1} - \bar \cL_A))\exp(\hat\gamma c_0\|\lambda_{T_1}\|_1/(3\nu))}_{=: \exp(r)}.
    \end{align*}
    Applying $\ln$ and rearranging,
    \begin{align*}
        \frac {- b_t^\top A\lambda_t}{\|\lambda_t\|_1}
        \geq \frac {\hat \gamma c_0}{3\nu}
        - \frac {r}{\|\lambda_{t}\|_1}.
    \end{align*}
    To finish, by \Cref{fact:general:margins:big_norms}, there exists $T_2$ so that
    \[
        \|\lambda_i\|_1 \geq \frac {6r\nu}{\hat\gamma c_0}
    \]
    for every $i\geq T_2$,
    and setting $T := \max\{T_1, T_2\}$ gives the desired result.
\end{proof}

\begin{lemma}
    \label{fact:general:margins:wolfe}
    Consider the setting of \Cref{fact:general:margins},
    but specialized so that $\alpha_i \in \alwol{i}$.
    Then there exists $\hat\gamma>0$ and $T$ so that,
    for all $t\geq T$, all margins exceed $\hat\gamma$.
\end{lemma}
\begin{proof}[Proof sketch]
    Choose $T_1$ according to
    \Cref{fact:general:margins:hatgamma}, and let
    $t\geq T_1$ be arbitrary.
    Using \Cref{fact:general:margins:hatgamma},
    and using the first Wolfe condition
    (\cref{eq:wolfe:1}) just as in the proof of \Cref{fact:sep:margin:wolfe},
    \begin{align*}
        &\cL(A\lambda_{t+1}) - \bar \cL_A
        \\
        &\leq \cL(A\lambda_t) - \bar \cL_A
        - \alpha_{t+1} (1-\nu/2) \|A^\top \nabla \cL(A\lambda_t)\|_\infty
        \\
        &=( \cL(A\lambda_t) - \bar \cL_A)\left(
        1 -
        \frac{\alpha_{t+1} (1-\nu/2) \|A^\top \nabla \cL(A\lambda_t)\|_\infty}
        {\cL(A\lambda_t)- \bar \cL_A}
        \right)
        \\
        &\leq ( \cL(A\lambda_t) - \bar \cL_A)\left(
        1 -
    \alpha_{t+1} (1-\nu/2)\hat \gamma\right)
        \\
        &\leq ( \cL(A\lambda_t) - \bar \cL_A)\exp\left(
    - \alpha_{t+1} (1-\nu/2)\hat \gamma\right)
    \end{align*}
    Applying this inequality recursively,
    \begin{align*}
        &\cL(A\lambda_{t+1}) - \bar \cL_A
        \\
        &\leq
        (\cL(A\lambda_{T_1} - \bar \cL_A))\exp\left(
            -(1-\nu/2) \hat \gamma \sum_{i=T_1}^{t+1} \alpha_i
        \right).
    \end{align*}
    Note next, for any $t_0$, that
    \[
        \|\lambda_{t+1}\|_1
        \leq \|\lambda_{t_0}\|_1 + \sum_{i = t_0+1}^{t+1} \alpha_i,
    \]
    whereby
    \begin{align*}
        &\cL(A\lambda_{t+1}) - \bar \cL_A
        \\
        &\leq
        (\cL(A\lambda_{T_1} - \bar \cL_A))
        \\
        &\quad\cdot
        \exp\left(
            -(1-\nu/2) \hat \gamma (\|\lambda_{t+1}\|_1 - \|\lambda_{t_0}\|_1)
        \right),
    \end{align*}
    and the remainder of the proof proceeds just as for the quadratic upper bound
    (cf. \Cref{fact:general:margins:quadub}).
\end{proof}

\begin{proof}[Proof sketch of \Cref{fact:general:margins}]
    The case of $\alqub{t}$ and $\alwol{t}$ are handled
    by \Cref{fact:general:margins:quadub} and \Cref{fact:general:margins:wolfe}.

    Now consider the case of $\alada{t}$.  Since $\gamma = 0$,
    \Cref{fact:general:gammat_shrinks}
    grants the existence of a large $T$ so that, for all $t\geq T$,
    $\gamma_t \leq 0.1$.  Thus, by
    \Cref{fact:quadub_vs_ada}, and considering $t$ sufficiently large that
    $C_t$ is almost 1, the problem reduces to the consideration
    of $\alqub{t}$; in particular, the conditions to apply
    \Cref{fact:general:margins:quadub}, but now for the step $\alada{t}$, are satisfied.
    Note that this also handles the case $\alopt{t}$, since,
    for $\alopt{t}$ and $\alada{t}$, it was assumed that $A$ is binary and $\ell = \exp$.
    %yeah, the following doesn't work; additive error not good enough.
  %%Similarly, the case of $\alopt{t}$ follows from the result for $\alada{t}$
  %%by an application of \Cref{fact:alopt_almost_alada}.
\end{proof}

\subsection{Miscellaneous Technical Material}
\begin{lemma}
    \label{fact:quadub_vs_ada}
    For any $r\in [0,1)$,
    \[
        r \leq \frac 1 2 \ln\left(\frac {1+r}{1-r}\right) \leq \frac r{1-r}.
    \]
\end{lemma}
\begin{proof}
    Set $g(r) := \frac 1 2 \ln((1+r)/(1-r))$.  Note that
    \[
        g'(r) = (1-r^2)^{-1}
        \qquad
        \textup{and}
        \qquad
        g''(r) = \frac {2r}{(1-r^2)^2}.
    \]
    As such, $g$ is convex (along $[0,1)$) and $g'(0) = 1$, thus $g(r) \geq r$
    along $[0,1)$.  The second part follows from concavity of $\ln(\cdot)$:
    \[
        \frac 1 2 \ln\left(\frac {1+r}{1-r}\right)
        = \frac 1 2 \ln\left(1 + \frac {2r}{1-r}\right)
        \leq \frac 1 2 \left(\frac {2r}{1-r}\right).
        %= \frac {2r}{1-r}.
        \qedhere
    \]
\end{proof}

\begin{lemma}
    \label{fact:general:gammat_shrinks}
    Under the conditions of \Cref{fact:general:margins},
    $\lim_{t\to\infty} \gamma_t = 0$.
\end{lemma}
\begin{proof}[Proof sketch]
    As discussed in the proof of \Cref{fact:general:opt},
    every step size provides a guarantee of the type
    \[
        \cL(A\lambda_{t+1}) \leq \cL(A\lambda_{t}) -\frac{\gamma_t^2\cL(A\lambda_t)}{c}
    \]
    for some $c>0$ (independent of $t$).
    The result follows by rearranging this expression
    and using $\cL(A\lambda_t) \geq \bar\cL_A > 0$ (i.e., nonseparability)
    and $\cL(A\lambda_{t} - \cL(A\lambda_{t+1}) \to 0$
    (i.e., the convergence result, \Cref{fact:general:opt}).
\end{proof}

\end{document}